\documentclass{article}
\usepackage{iclr2026_conference,times}

\usepackage{amsmath,amsfonts,bm}

\def\eqref#1{equation~\ref{#1}}
\def\1{\bm{1}}

\DeclareMathAlphabet{\mathsfit}{\encodingdefault}{\sfdefault}{m}{sl}
\SetMathAlphabet{\mathsfit}{bold}{\encodingdefault}{\sfdefault}{bx}{n}

\newcommand{\Dinv}{D_{\psi_{\mathrm{inv}}}}

\newcommand{\alphainv}{\alpha_{\mathrm{inv}}}
\newcommand{\beff}{\beta_{\mathrm{eff}}}
\newcommand{\blow}{\beta_{\mathrm{low}}}
\newcommand{\bhigh}{\beta_{\mathrm{high}}}

\newcommand{\concc}{\mathrm{conc}(c)}
\newcommand{\rinv}{r_{\mathrm{inv}}}
\newcommand{\rlag}{r_{\mathrm{lag}}}
\newcommand{\rinve}{r_{\mathrm{inv},e}}
\newcommand{\rlage}{r_{\mathrm{lag},e}}
\newcommand{\Tinv}{T_{\mathrm{inv}}}
\newcommand{\Tlag}{T_{\mathrm{lag}}}
\newcommand{\Ttrust}{T_{\mathrm{trust}}}
\newcommand{\Tfocus}{T_{\mathrm{focus}}}

\newcommand{\zinv}[1]{z_{#1}^{I}}
\newcommand{\zlewm}[1]{z_{#1}^{L}}
\newcommand{\zinve}[2]{z_{#1}^{#2,I}}
\newcommand{\zlewme}[2]{z_{#1}^{#2,L}}

\newcommand{\zS}{z_S}

\newcommand{\abase}{a_{\mathrm{base}}}

\newcommand{\Hmacro}{H_{\mathrm{macro}}}
\newcommand{\Hsub}{H_{\mathrm{sub}}}

\usepackage{hyperref}
\usepackage{url}
\usepackage{amsmath}
\usepackage{amssymb}
\usepackage{amsthm}
\usepackage{booktabs}
\usepackage{graphicx}
\usepackage{placeins} % \FloatBarrier: keep Method figures within their subsections
\usepackage{enumitem}
\usepackage{tikz} % retained for the forthcoming regenerated system-overview figure

\theoremstyle{plain}
\newtheorem{theorem}{Theorem}[section]
\newtheorem{lemma}{Lemma}[section]

\theoremstyle{definition}

\theoremstyle{remark}
\newtheorem{remark}{Remark}[section]

\title{IMWM: Intuition Models Complement \\ World Models for Latent Planning}

\author{Baoqi Gao\textsuperscript{1,2}, \ Ruize Han\textsuperscript{2}, \ Miao Wang\textsuperscript{1}, \ Song Wang\textsuperscript{2} \\[4pt]
\textnormal{\textsuperscript{1}Beihang University \qquad
\textsuperscript{2}Shenzhen University of Advanced Technology}}

\iclrfinalcopy % arXiv version: reveal authors and drop the submission line-number ruler.

\begin{document}

\maketitle
\lhead{Preprint} % arXiv preprint: override the ICLR final-copy "Published as..." banner.

% IMWM paper section files. Old pre-rewrite files (introduction.tex,
% method.tex, etc., under sections/) are preserved on disk for internal
% reference but are not \input'd.
\begin{abstract}
Planning with a learned latent world model is a promising route to
control from raw pixels, but a strong world model alone is not enough.
We show this experimentally: even with a perfect world model
(operationalized by replacing the learned forward predictor with an
idealized rollout of the true environment dynamics), a finite-budget
sample-based planner still fails on some tasks, indicating that the
bottleneck can lie in search rather than in world-model accuracy.
Motivated by this gap, we propose \emph{IMWM} (\emph{Intuition Model +
World Model}), which pairs the world model with an intuition model
trained from demonstrations to recognize promising actions. The two
models collaborate through three lightweight components:
(i)~Retrieval Initialization, which initializes the planner's
action proposal from a retrieved demonstration; (ii)~Hybrid Cost,
which combines the intuition score with the world-model rollout cost;
and (iii)~a Reliability Gate, which adjusts how much the planner
trusts intuition in each setting. Across four pixel-based goal-reaching
tasks (Two-Room, Reacher, Push-T, and OGBench-Cube), IMWM has higher mean
success than the world-model-only planner on all four, with the largest gains
on Two-Room
($99.2\%$, $+11.5$ percentage points) and OGBench-Cube ($94.7\%$, $+28.5$
percentage points).
\end{abstract}

\section{Introduction}
\label{sec:intro}

Planning with a learned latent world model is a competitive route to
control from raw pixels: an encoder maps each image to a compact latent
state, a learned forward predictor rolls that state under candidate
actions, and a sample-based optimizer such as the Cross-Entropy
Method~\citep{rubinstein99cem} or its modern
variants~\citep{pinneri20icem,tdmpc} ranks action sequences under a
latent cost (e.g.\ terminal latent-MSE to a goal embedding). The lineage
runs from visual MPC~\citep{ebert18visualforesight} through the
PlaNet/Dreamer family~\citep{planet,dreamerv1,dreamerv2,dreamerv3} and
PETS~\citep{pets} to modern latent MPC such as TD-MPC~\citep{tdmpc}. The
default way to improve such a planner is to improve the world
model. This paper studies when that instinct is insufficient.

Planning has two ingredients: a model that predicts outcomes, and a
search that proposes which actions to try. We show \emph{experimentally}
that the search can be the binding constraint. We operationalize a
perfect world model by replacing the learned forward predictor with an
idealized rollout of the true environment dynamics, while keeping the
terminal latent-MSE cost and the original CEM budget fixed. Even so, on these
12-cell diagnostic grids, this oracle-dynamics planner still trails IMWM by $17.9$~pp on
OGBench-Cube and $13.7$~pp on Two-Room, and its failures are \emph{search} failures:
$127/139$ failed OGBench-Cube episodes and $86/87$ failed Two-Room episodes
contain \emph{zero} goal-reaching candidates in the CEM population at both logged
replans. The bottleneck is where the finite queries land, not how the world model
predicts. We also prove this theoretically: under finite CEM queries
against any black-box latent cost, the planner's success probability
obeys a proposal-volume bound that is independent of how the predictor
was trained (Theorem~\ref{thm:bottleneck}, Appendix~\ref{app:formal}).

\begin{figure}[t]
\centering
\includegraphics[width=\textwidth]{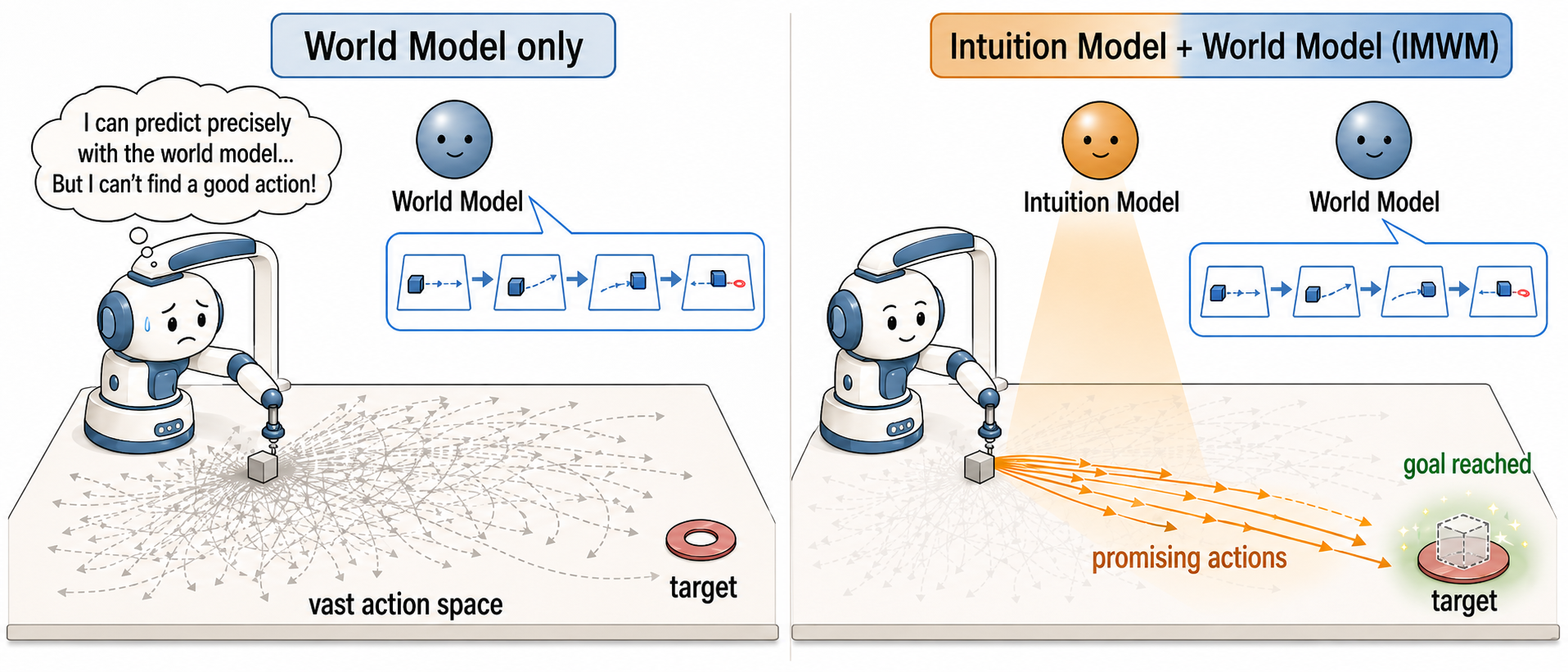}
\caption{\textbf{A world model alone is not enough.} \emph{Left (World Model only):}
the planner rolls out candidate actions accurately, but under a finite search
budget it cannot find a goal-reaching action in the vast action space.
\emph{Right (IMWM):} pairing the same frozen world model with a frozen
\emph{intuition model} concentrates the search on the promising actions and
reaches the goal. Both regimes use the same frozen world model and the same CEM
budget; only the proposal, cost, and gating differ.}
\label{fig:hero}
\end{figure}

If the search is what fails, it needs better guidance about which
actions are worth trying (Figure~\ref{fig:hero}). We take loose inspiration from how biological
decision-making works: not as a single forward simulator, but as several
functional motifs whose influence is regulated rather than fixed. These motifs
include arbitration between control systems~\citep{daw05,lengyel07},
memory-driven retrieval and predictive
maps~\citep{mattar18,pfeiffer13,stachenfeld17}, and hierarchical action
chunking~\citep{dezfouli13}. We make no claim about biological mechanism; we use these motifs
only as a design vocabulary, with established machine-learning analogues
in episodic control~\citep{blundell16mfec,pritzel17nec}, decision-time
retrieval and imitation~\citep{goyal22rarl,qi22implant}, behavior and
trajectory priors~\citep{pertsch20spirl,tirumala22bp,janner22diffuser},
and value-aware model
learning~\citep{farahmand17vaml,lambert20_objective_mismatch}.
\S\ref{sec:motivation} positions IMWM against this work, and
Appendix~\ref{app:extended-positioning} gives the full ledger.

We propose \emph{IMWM} (\emph{Intuition Model + World Model}), in which an
intuition model complements rather than replaces the world model. The
intuition model $\Dinv$ is a contrastive scalar score over (start
latent, goal latent, action chunk) tuples, trained on demonstration
windows with an InfoNCE-style objective~\citep{vandenoord18cpc} to
recognize promising actions. Its encoder and scorer feed three
lightweight collaboration components.
\emph{(i) Retrieval Initialization} centers the
CEM proposal on the action chunk retrieved by nearest-neighbor lookup in
the intuition encoder's latent space, biasing queries toward
demonstration-supported regions. \emph{(ii) A Hybrid Cost}
combines the intuition score with the world-model rollout cost.
\emph{(iii) A Reliability Gate} adjusts, per setting, how much the
planner trusts intuition versus falling back to the world model alone.
The algorithmic family and all trained artifacts are fixed across
experiments, under frozen thresholds; the only per-setting decision is
which of three pre-specified recipes the gate selects, so improvements
cannot come from per-task tuning (\S\ref{sec:method}).

\paragraph{Contributions.}
\begin{enumerate}
\item \textbf{Diagnosis of finite-CEM latent-WM planning.} We show
\emph{experimentally}, with an oracle-dynamics ablation that replaces the
learned predictor with idealized true-dynamics rollout under the original
CEM budget, that an idealized world model still fails on OGBench-Cube and
Two-Room because the search, not the predictor, is the binding
constraint (\S\ref{sec:diagnosis-oracle}); we also prove a matching
predictor-independent proposal-volume bound
(Theorem~\ref{thm:bottleneck}, Appendix~\ref{app:formal}).
\item \textbf{The IMWM method.} We design and train an \emph{intuition model}
that recognizes promising actions, and integrate it with
the world model through three components: Retrieval
Initialization, a Hybrid Cost, and a Reliability Gate
(\S\ref{sec:method}).
\item \textbf{Empirical gains on four pixel-based goal-reaching tasks.}
Over a 48-cell fresh-seed grid (12 cells $\times$ 4 tasks), IMWM achieves
higher mean success than the world-model-only planner on all four tasks:
Two-Room $99.2\%$ ($+11.5$~pp, $12/0/0$ W/T/L), OGBench-Cube $94.7\%$
($+28.5$~pp, $12/0/0$), Push-T $92.7\%$ ($+2.8$~pp, $8/1/3$), and Reacher
$83.8\%$ (a near-tie, $+0.7$~pp, $7/1/4$), where the gate routes to the
forward-only fallback (\S\ref{sec:exp-main}).
\end{enumerate}

\paragraph{Roadmap.}
\S\ref{sec:diagnosis} establishes the finite-query diagnosis;
\S\ref{sec:motivation} positions IMWM against related work;
\S\ref{sec:method} defines IMWM; \S\ref{sec:experiments}
reports main results, routing, and ablations; and
\S\ref{sec:limitations} and \S\ref{sec:conclusion} discuss limitations
and conclude.
\section{Diagnosis: where latent world-model planning fails}
\label{sec:diagnosis}

The introduction framed planning as two ingredients (a model that predicts
and a search that proposes) and claimed they are separate failure surfaces.
Here we ask the sharp version: assuming the world model is correct, can a
finite-budget planner still fail? It can, and we show this
\emph{experimentally}. We swap the learned forward predictor for the
\emph{true} environment dynamics while holding everything else fixed. The
resulting oracle-dynamics planner barely improves: it even underperforms the
world-model-only baseline on one task, still trails IMWM, and its failures are
\emph{search} failures rather than prediction failures
(\S\ref{sec:diagnosis-oracle}). The same conclusion holds in the worst case as a
theorem: finite black-box search is volume-limited for reasons independent of
how the predictor was trained. We state and prove it in
Appendix~\ref{app:formal} (Theorem~\ref{thm:bottleneck}).

\subsection{Oracle dynamics: an idealized world model still fails}
\label{sec:diagnosis-oracle}

To test whether the search, rather than the forward predictor, is the limiting
component in practice, we replace the learned predictor with literal environment
rollout while holding everything else fixed: the same encoder, terminal
latent-MSE cost, CEM budget, replanning cadence, and evaluation cells
(configuration in Appendix~\ref{app:cem-hyperparams}). We call this
\emph{oracle dynamics}: it changes only the latent dynamics used in the cost,
not the cost or the budget.

% NOTE: oracle / world-model-only / IMWM columns below are now full 12-cell
% (ds,ss) grid means with the SAME sourcing as the Sec. 5 headline table, so the
% world-model-only (87.7 / 66.2) and IMWM (99.2 / 94.7) columns equal the Sec. 5
% headline means. Oracle = mean over the 12 oracle-dynamics cells per task.
% Source: paper/iclr2026/oracle_lewm_fullgrid.json (12/12 cells, both tasks).
\begin{table}[h]
\centering
\caption{Oracle dynamics versus the world-model-only baseline (provenance tag
\emph{c1}) and IMWM, all under the same finite CEM budget. Oracle dynamics
replaces the learned forward predictor with literal environment rollout in the
terminal latent-MSE cost; everything else is unchanged. \emph{Zero-success
failed eps} counts failed episodes for which every logged replan had zero
goal-reaching candidates in its CEM population. Each task is the mean over its
full 12-cell $(\mathrm{ds},\mathrm{ss})$ grid; the world-model-only and IMWM
columns therefore coincide with the headline means of \S\ref{sec:experiments}.}
\label{tab:diagnosis-oracle}
\small
\resizebox{\textwidth}{!}{%
\begin{tabular}{lccccc}
\toprule
Task & Oracle dynamics & World-model only & IMWM & $\Delta$ (Oracle$-$IMWM) & Zero-success failed eps \\
\midrule
Two-Room (12-cell mean)     & 85.5 & 87.7 & 99.2 & $-13.7$\,pp & 86\,/\,87 (98.9\%) \\
OGBench-Cube (12-cell mean) & 76.8 & 66.2 & 94.7 & $-17.9$\,pp & 127\,/\,139 (91.4\%) \\
\bottomrule
\end{tabular}}
\end{table}

Table~\ref{tab:diagnosis-oracle} tells a consistent story. An exact forward
predictor \emph{underperforms} the world-model-only baseline on Two-Room
($87.7 \to 85.5$) and helps only modestly on OGBench-Cube ($66.2 \to 76.8$), and
in both cases still trails IMWM by $13.7$\,pp and $17.9$\,pp; an exact predictor
inside the original budget is no substitute for IMWM. The mechanism column
localizes why: failures are search failures. Among failed episodes, $98.9\%$ on
Two-Room ($86/87$) and $91.4\%$ on OGBench-Cube ($127/139$) had \emph{zero}
goal-reaching candidates in the CEM population at both replans, and every failed
episode on both tasks had at least one zero-candidate replan. When a replan's population does contain goal-reaching candidates, by contrast,
they make up $\sim$93--95\% of it, and the terminal latent-MSE objective almost
always ranks one of them first (Appendix~\ref{sec:exp-diag-rank-success}). This is
exactly the predictor-independent bottleneck of Theorem~\ref{thm:bottleneck}
showing up empirically, and it is the search-side gap that the method of
\S\ref{sec:method} is built to close.

\paragraph{Scope.}
This is a two-task experiment through a single oracle adapter. On these
diagnostic cells, it rules out the specific explanation that the forward
predictor's accuracy is the bottleneck under the original budget and cost; it is
not a claim about arbitrary perfect models, alternative costs, or larger budgets. The
full CEM-budget configuration is in Appendix~\ref{app:cem-hyperparams}, and
Push-T is excluded by a simulator-physics confound documented in
Appendix~\ref{app:pusht-quarantine}; Reacher is excluded because its
world-model-only baseline is already competitive (a paired near-tie of
$+0.7$\,pp; \S\ref{sec:experiments}). The few failures that do not fit the
both-replans-zero pattern ($12$ of $139$ on OGBench-Cube, $1$ of $87$ on
Two-Room) are the cases where a goal-reaching candidate appeared in some replan
yet was not selected.

\section{Related work and positioning}
\label{sec:motivation}

IMWM draws on several established lines of work, but its contribution is a
specific \emph{combination}, not any single ingredient. We position it
against the closest lines here; the full related-work ledger and the
cognitive-science motifs that inspired pairing an intuition model with a
world model are in Appendix~\ref{app:extended-positioning}.

\paragraph{What is new.}
IMWM's novelty is this combination, realized inside a single finite-budget
planner: (i) a frozen latent world model with terminal latent-MSE as one
cost term, (ii) a separately trained contrastive (start, goal, action)
compatibility score as a second cost term, (iii) cosine retrieval of a
demonstration action chunk to center the CEM proposal, and (iv) a per-cell
reliability gate that selects among these ingredients \emph{before any query
is spent}. The diagnosis of \S\ref{sec:diagnosis} identifies the precise
search bottleneck this combination is built to attack.

\paragraph{Positioning against the closest lines.}
\emph{Sampling-based planners.} iCEM~\citep{pinneri20icem} improves CEM
through colored-noise action correlation and elite memory across replans;
IMWM instead conditions the proposal on a demonstration-retrieved mean and a
learned compatibility cost, a different and complementary axis. Modern latent
MPC such as TD-MPC~\citep{tdmpc} learns a task-oriented dynamics model and a
terminal value; IMWM keeps the world model frozen and reward-free, adding the
intuition signal only as a cost term.
\emph{Contrastive goal-conditioned value.} Contrastive RL~\citep{eysenbach22contrastiverl}
interprets a contrastive start-goal-action score as a goal-conditioned value
function; IMWM uses an analogous score not as a value but as a planner-side
compatibility cost inside finite-budget CEM.
\emph{Episodic and retrieval-based control.} MFEC and
NEC~\citep{blundell16mfec,pritzel17nec} retrieve stored \emph{values} for
$\arg\max$ action selection; IMWM retrieves \emph{action chunks} to initialize
the search distribution, with no value memory.
\emph{Value-aware model learning.} Objective-mismatch and value-aware
losses~\citep{farahmand17vaml,lambert20_objective_mismatch} reshape the
\emph{training} objective so that prediction better serves control; our
diagnosis is orthogonal: it exposes a \emph{predictor-independent},
finite-query search bottleneck that better training alone cannot remove.

\section{Method: IMWM}
\label{sec:method}

IMWM augments standard latent world-model planning with a frozen
\textbf{intuition model}: an inverse-side encoder $\phi^I$ and a bilinear
scorer $\Dinv$ over (start, goal, action) triples, jointly trained on
demonstration windows by an InfoNCE-style objective~\citep{vandenoord18cpc}
and frozen at evaluation. The planner consumes the intuition model through
\emph{three components}: Retrieval Initialization centers the search
proposal on a retrieved demonstration chunk, Hybrid Cost combines the
intuition score with the world-model rollout cost, and a Reliability
Gate decides per cell how much to rely on the intuition model versus falling
back to the world model alone. The empirical gains we report are for this
composed, gated planner; the ablations in \S\ref{sec:experiments} separate the
contributions of these components while leaving the main empirical claim on the
full gated method. The algorithmic family, trained artifacts, and retrieval banks are fixed before
evaluation. The gate's only per-cell decision is which of three pre-specified
recipes to use, so improvements cannot come from per-task tuning. Figure~\ref{fig:arch} shows the
overall architecture.

\begin{figure}[t]
\centering
\includegraphics[width=\textwidth]{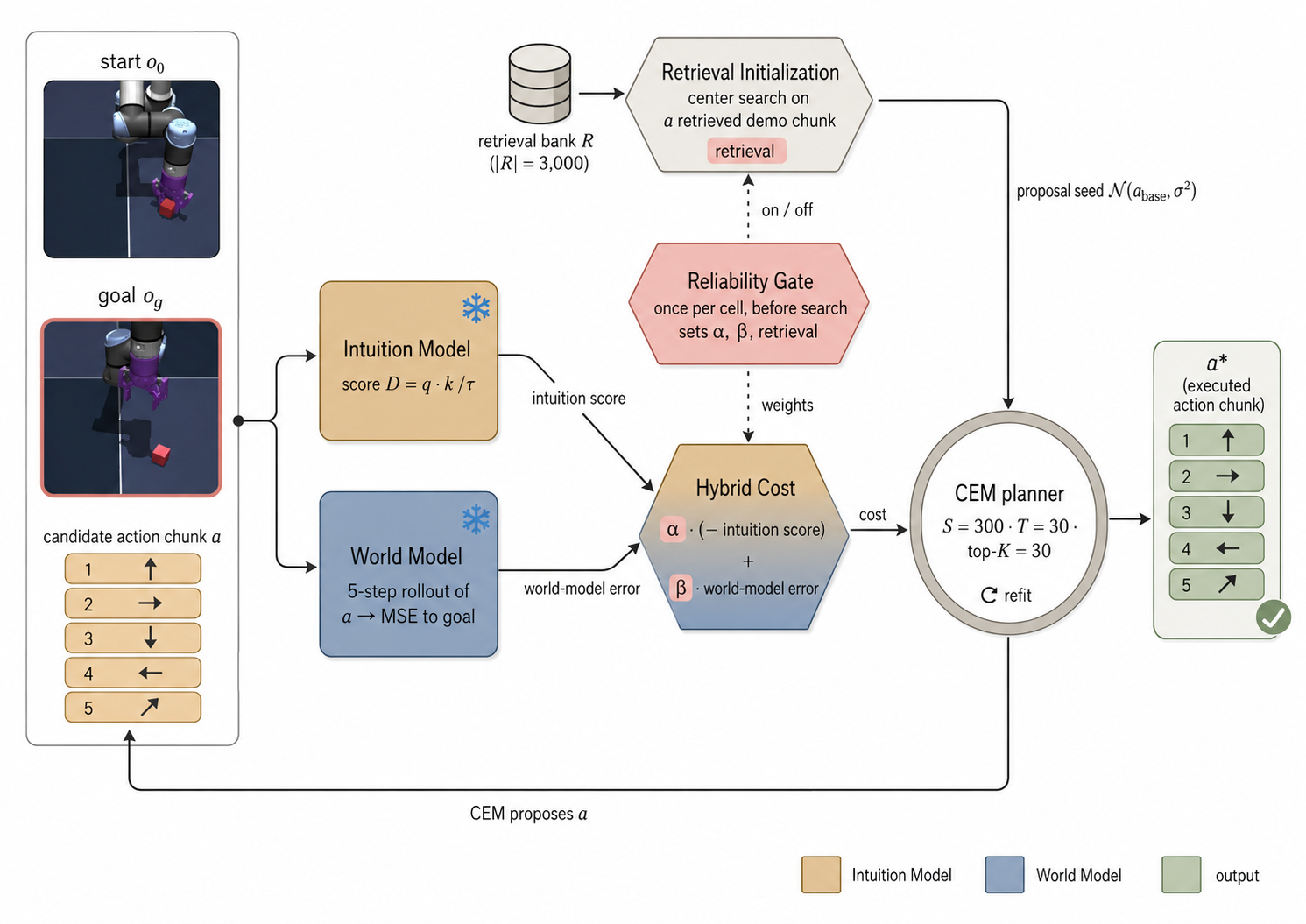}
\caption{\textbf{IMWM architecture.} Two frozen models encode the start/goal
observations and score a candidate action chunk: the \emph{Intuition Model}
(compatibility $D = q\cdot k/\tau$) and the \emph{World Model} (a 5-step latent
rollout scored by terminal MSE to the goal). Their outputs are combined by the
\emph{Hybrid Cost} $J$, a weighted sum (weights $\alpha,\beta$) of the
$z$-scored \emph{negative} intuition score ($-D$) and the $z$-scored world-model
error.
\emph{Retrieval Initialization} seeds the CEM proposal from a demonstration
bank, and the \emph{Reliability Gate} sets the per-cell recipe (retrieval
on/off and the weights $\alpha,\beta$). For each CEM candidate the Hybrid Cost
queries both frozen models; CEM returns the executed action chunk $a^{*}$. The
start/goal panels are real OGBench-Cube observations.}
\label{fig:arch}
\end{figure}

\FloatBarrier

\subsection{Setup and the intuition model}
\label{sec:method-setup}

We consider goal-reaching MDPs over pixel observations, evaluated in
\emph{cells} $c = (\mathrm{task}, ds, ss)$ of $E = 50$ episodes. Each
environment $e$ provides start/goal observations $(o_0^e, o_g^e)$, encoded by
two frozen encoders: the intuition encoder $\phi^I$ producing
$(\zinve{0}{e}, \zinve{g}{e})$ in $z^I$-space, and the world model (instantiated
with the LeWM stack of~\citet{lewm}, run on the \texttt{stable-worldmodel} platform~\citep{swm}) producing $(\zlewme{0}{e}, \zlewme{g}{e})$
in $z^L$-space with an $\Hmacro{=}5$-step latent rollout. The
\emph{world-model-only baseline} is this stack planned by Cross-Entropy Method
(CEM) without the intuition model. All planners we compare share the same frozen
world model and CEM budget and differ only in proposal, cost, and per-cell
conditioning. The intuition scorer is bilinear,
$\Dinv(\zinve{0}{e}, \zinve{g}{e}, a) = q(\zinve{0}{e}, \zinve{g}{e})\cdot k(a)/\tau$;
larger values mean the chunk $a$ is more compatible with the start--goal pair
(Figure~\ref{fig:scoring}); the scorer is trained per task with an InfoNCE
objective (Figure~\ref{fig:training}).
Architecture, the InfoNCE objective, and bank/training details are in
Appendix~\ref{app:components}.

\begin{figure}[tbp]
\centering
\includegraphics[width=0.86\textwidth]{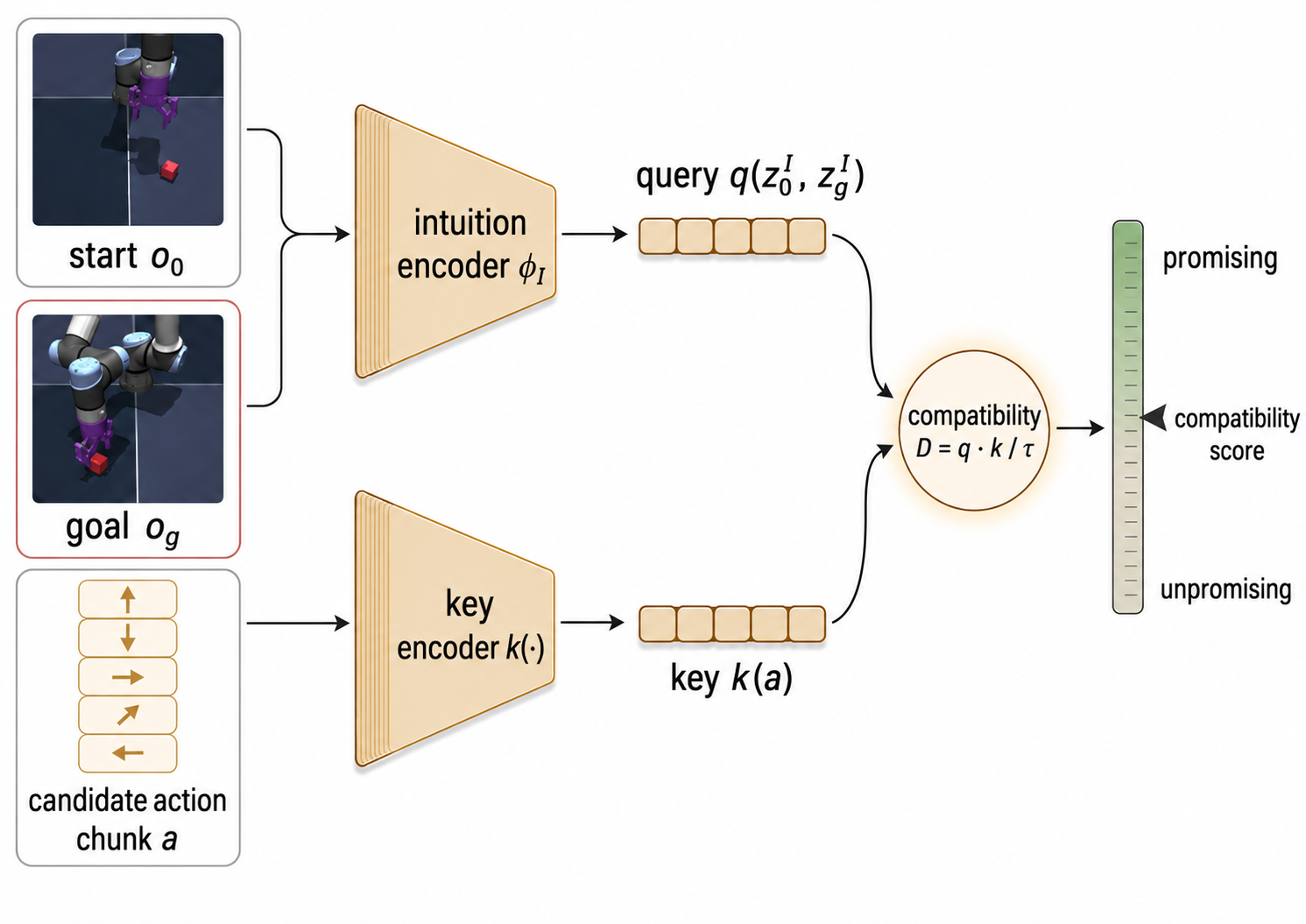}
\caption{\textbf{Intuition-model scoring.} The (start, goal) observations are
encoded into a query $q(z_0^{I}, z_g^{I})$ and a candidate action chunk into a
key $k(a)$; their compatibility $D = q\cdot k/\tau$ scores how well the action
fits the start$\to$goal transition (higher is more promising). The start/goal
panels are real OGBench-Cube observations; the action chunk is drawn
schematically.}
\label{fig:scoring}
\end{figure}

\begin{figure}[tbp]
\centering
\includegraphics[width=0.80\textwidth]{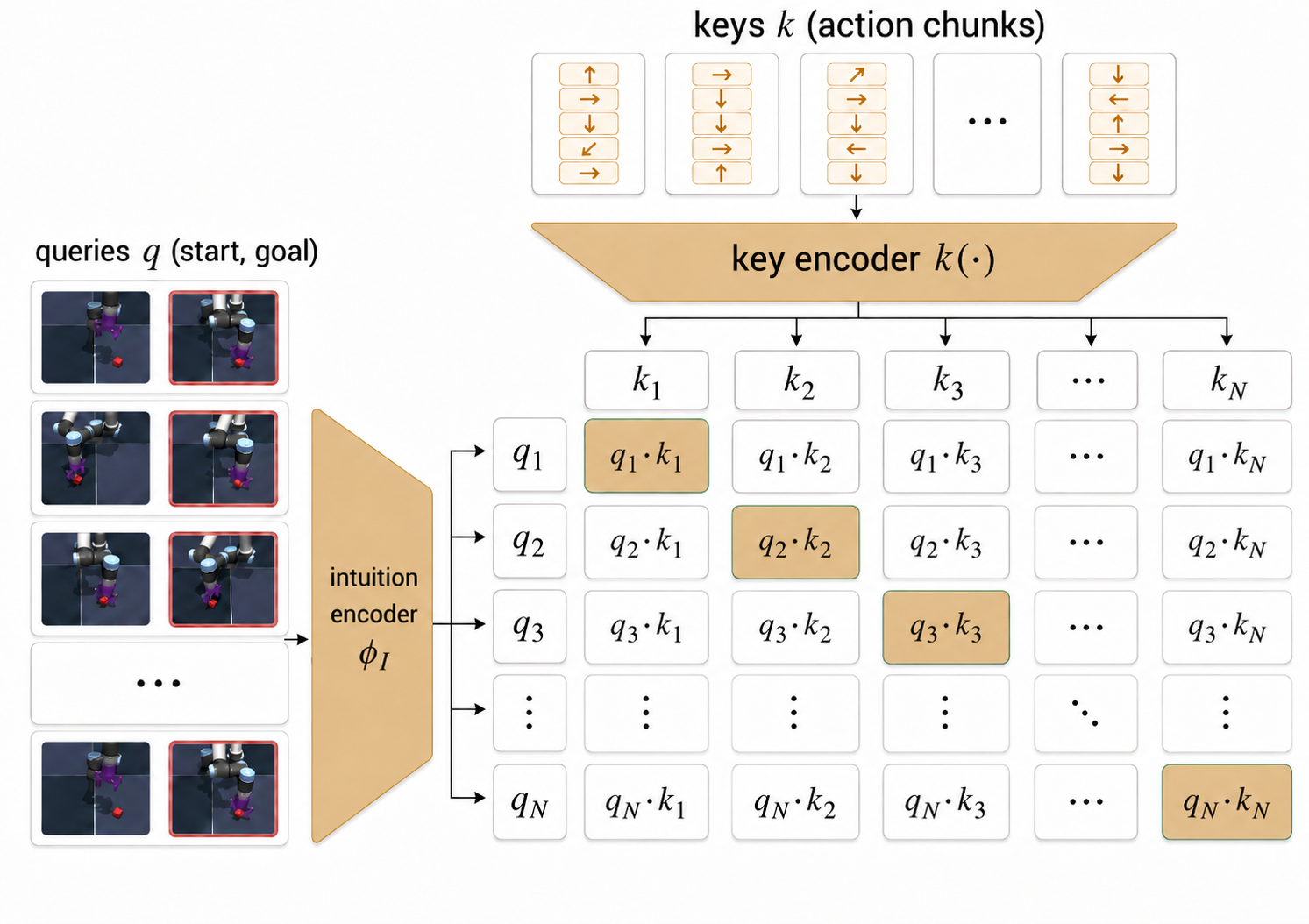}
\caption{\textbf{Intuition-model training (per task, CLIP-style).} Rows are
start--goal queries and columns are schematic action-chunk keys for one
OGBench-Cube task-specific batch. Diagonal cells pair each query with its own
action chunk and are the InfoNCE positives; off-diagonal cells pair the query
with other same-task action chunks and are in-batch negatives. The real
OGBench-Cube thumbnails are illustrative observation examples; they are not
claimed to be the exact training-bank samples.}
\label{fig:training}
\end{figure}

\paragraph{Why the intuition model complements the world model.}
The intuition model is deliberately a counterpart, not a replacement, to the
forward predictor, and four choices keep them complementary.
\emph{(i)} It is \emph{inverse-conditional}: it scores actions given
start--goal evidence rather than predicting the next latent.
\emph{(ii)} It is trained on \emph{demonstration windows}, not the
$(s,a,s')$ transitions the world model trains on.
\emph{(iii)} It returns a \emph{scalar compatibility score}, not a latent
prediction, so it does not compound rollout error.
\emph{(iv)} It lives in a \emph{separate latent space} $z^I$, decoupling the
intuition representation from the forward-prediction one. These choices echo
inverse-dynamics representation learning~\citep{pathak17icm} and contrastive
goal-conditioned scoring~\citep{eysenbach22contrastiverl}; the precise
mechanism differences from each are in Appendix~\ref{app:extended-positioning}.

\subsection{Three components}
\label{sec:method-components}

\paragraph{(i) Retrieval Initialization.}
Following the diagnosis (\S\ref{sec:diagnosis}), IMWM injects prior information
about plausible actions \emph{before} the cost rules them out, by centering the
CEM proposal on a retrieved demonstration chunk. From a fixed bank
$R = \{(k_0^{(i)}, k_g^{(i)}, a^{(i)})\}_{i=1}^{|R|}$ of cosine-normalized
demonstration start/goal keys and action chunks ($|R| = 3{,}000$ per
task/data-seed; built once, Appendix~\ref{app:retrieval-bank}), the planner
retrieves
\begin{equation}
\label{eq:retrieval-init}
\abase^{(e,t)} = a^{(i^*)}, \quad
i^* = \arg\max_{i}\, \mathrm{cos}\!\big([\zinve{t}{e}; \zinve{g}{e}],\,
[k_0^{(i)}; k_g^{(i)}]\big),
\end{equation}
and initializes the iter-$0$ CEM candidates as an isotropic Gaussian
$\mathcal{N}(\abase^{(e,t)}, \sigma^2 I)$ with $\sigma = 1$, sampled IID across
the $S = 300$ candidates and chunk dimensions, so there is no temporal
(colored-noise) correlation; the only injected structure is the retrieved mean.
When the gate disables retrieval, the proposal centers on the zero action chunk,
recovering the standard world-model-only CEM proposal.

\paragraph{(ii) Hybrid Cost.}
The candidate cost composes the intuition score with the world-model rollout
error:
\begin{equation}
\label{eq:hybrid-cost}
J(a; e, t) = \alphainv \cdot \zS\!\big(-\Dinv(\zinve{t}{e}, \zinve{g}{e}, a)\big)
\;+\; \beff \cdot \zS\!\big(\mathrm{MSE}(\mathrm{rollout}_L(\zlewme{t}{e}, a), \zlewme{g}{e})\big),
\end{equation}
where $\mathrm{rollout}_L$ is the $\Hmacro{=}5$-step world-model rollout and
$\mathrm{MSE}$ is the terminal latent-MSE to the goal. The two terms live in
different latent spaces on incommensurate scales, so each is $z$-scored
($\zS$) across the CEM candidate set before composition (formula in
Appendix~\ref{app:iwm-procedure}). The weights $(\alphainv, \beff)$ are not
free hyperparameters; the gate sets them from one of three recipes, and
$\alphainv = 0$ reduces $J$ to the world-model-only cost.

\paragraph{(iii) Reliability Gate.}
\label{sec:method-gate}
The intuition signal is not universally useful, so IMWM decides \emph{per cell,
before any planning query is spent}, how much to rely on it. At the first
replan it computes two diagnostics from the intuition scorer applied to the
retrieved chunk $\abase^{(e,0)}$ versus a fixed set of random ``neutral''
chunks: $\rinve$, the retrieved chunk's score margin over a high
($q_{95}$) percentile of the neutral scores, normalized by their dispersion;
and $\rlage$, the lag-1 autocorrelation of the retrieved chunk (a smoothness
diagnostic, not a sampling parameter). The cell features are the medians
$\rinv(c), \rlag(c)$ over environments (exact formulas in
Appendix~\ref{app:iwm-procedure}). Frozen thresholds
$(\Tinv, \Tlag) = (0.05, 0.3)$ select one of three recipes
$(\alphainv, \beff, \text{retrieval})$:
\[
\mathrm{recipe}(c) =
\begin{cases}
(1,\, 3,\, \text{on}), & \rinv(c) > \Tinv,\ \rlag(c) > \Tlag
  \quad\text{(trusted hybrid)}, \\[2pt]
(1,\, 0.1,\, \text{on}), & \rinv(c) > \Tinv,\ \rlag(c) \le \Tlag
  \quad\text{(intuition-dominant)}, \\[2pt]
(0,\, 1,\, \text{off}), & \rinv(c) \le \Tinv
  \quad\text{(forward-only fallback)}.
\end{cases}
\]
In the fallback recipe the intuition model is switched off entirely and the
planner reverts to the world-model-only configuration under the same budget
(up to solver-seed variation). The
thresholds were frozen before the 48-cell headline evaluation; the calibration
protocol and cell overlap are reported in \S\ref{sec:experiments}. A continuous relaxation of
these discrete weights is analyzed (not deployed) in
Appendix~\ref{app:conc-mode-defs}.

\subsection{Cell-level execution}
\label{sec:method-pseudocode}

Per cell, IMWM runs the diagnostic pre-pass for every environment, takes the
medians, and fixes $\mathrm{recipe}(c)$. Then, per episode, it plans
closed-loop: at each replan it (re-)centers the proposal via
\eqref{eq:retrieval-init} if retrieval is on, runs CEM ($S = 300$ candidates,
$T = 30$ iterations, top-$K = 30$ elites) under the hybrid
cost~\eqref{eq:hybrid-cost} with the recipe's weights, and executes the final
elite-mean chunk. The full procedure, neutral-chunk diagnostic algebra, and
compute footprint (IMWM matches the world-model-only CEM sampling budget; the
added intuition cost is small relative to the rollout) are in
Appendix~\ref{app:iwm-procedure}.

\section{Experiments}
\label{sec:experiments}

\subsection{Setup}
\label{sec:exp-setup}

We evaluate IMWM against the \emph{world-model-only baseline} (the same frozen
LeWM stack planned by CEM, no intuition model) on four pixel-based
goal-reaching tasks: \textbf{Two-Room} (locomotion to a goal), \textbf{Reacher}
(end-effector reaching), \textbf{Push-T} (T-block manipulation), and
\textbf{OGBench-Cube} (block pushing); Figure~\ref{fig:tasks} shows, for each
task, a start observation and the successful terminal observation. The headline grid is \textbf{12 cells per task} (48
total): six data seeds $ds \in \{3,5,7,9,11,13\}$, with solver-seed
replication $ss \in \{42,1,2\}$ on $ds \in \{3,5,7\}$, each cell $E = 50$
episodes. Cells are evaluated \emph{paired}: IMWM and the baseline see the same
episodes under the same per-replan CEM budget ($S = 300$, $T = 30$,
top-$K = 30$), so the comparison equalizes sampling budget, not wall-clock
(approximate logged timings, with their batching caveat, are in
Appendix~\ref{app:runtime-audit}). We report cell-level success rate, the mean
paired delta $\bar{\Delta}$ (IMWM $-$ baseline), and paired wins/ties/losses
(tie iff $|\Delta| < 2$~pp).

\begin{figure}[t]
\centering
\includegraphics[width=\textwidth]{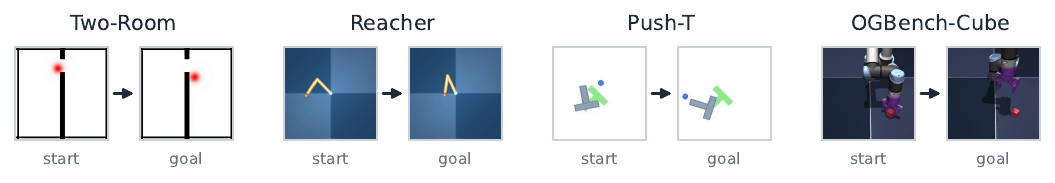}
\caption{\textbf{The four pixel-based goal-reaching tasks.} For each task we show
two real evaluation frames: the start observation $o_0$ and the
benchmark-provided goal observation $o_g$ (the goal image the planner is
given), from a representative episode (cell \texttt{ds3\_ss42}).
\emph{Two-Room}: traverse a doorway to a goal in the
adjacent room; \emph{Reacher}: move a jointed arm's end-effector to a target;
\emph{Push-T}: align a T-block to a target pose; \emph{OGBench-Cube}: push a
block to a target. Frames are rendered by the actual simulators (the planner
observes $224\times224$ pixels).}
\label{fig:tasks}
\end{figure}

\subsection{Main results}
\label{sec:exp-main}

Table~\ref{tab:headline} reports the 12-cell-per-task totals. IMWM
substantially improves over the world-model-only baseline on Two-Room
($99.2\%$, $+11.5$~pp, $12/0/0$) and OGBench-Cube ($94.7\%$, $+28.5$~pp,
$12/0/0$), with paired-cell bootstrap $95\%$ CIs well above zero; yields a small
but positive improvement on Push-T ($92.7\%$, $+2.8$~pp, $8/1/3$; CI excludes
zero but its lower bound is below $+2$~pp); and is a near-tie on Reacher
($83.8\%$, $+0.7$~pp, $7/1/4$; CI crosses zero), where the gate routes to the
forward-only fallback recipe (\S\ref{sec:exp-routing}), leaving IMWM
statistically tied with the baseline there (the $+0.7$~pp is run-to-run CEM variation under the
shared fallback recipe). The two large gains occur on the tasks where our diagnostics show
the clearest search limitation; the per-cell paired scatter is in
Appendix~\ref{app:extra-results} (Figure~\ref{fig:paired-scatter}).

\begin{table}[t]
\centering
\small
\caption{Fresh-seed 12-cell totals (success rate, \%; $n{=}12$ cells/task).
$\bar{\Delta}$ is the mean cell-level paired delta (IMWM $-$ baseline); $95\%$
CIs are paired-cell bootstrap intervals over the evaluated grid
($B{=}10{,}000$). W/T/L tie band is $|\Delta| < 2$~pp; at $50$ episodes/cell
this reduces to $\Delta = 0$.}
\label{tab:headline}
\begin{tabular}{l c c r c c}
\toprule
                  & \textbf{IMWM}   & \textbf{World-model only} & $\bar{\Delta}$ & 95\% CI (pp)            & paired \\
\textbf{Task}     & mean $\pm$ SD  & mean $\pm$ SD             & (pp)           & (paired-cell bootstrap) & W/T/L \\
\midrule
Two-Room      & $99.2 \pm 1.0$ & $87.7 \pm 4.7$ & $+11.5$ & $[+8.8, +14.5]$  & 12/0/0 \\
Reacher       & $83.8 \pm 5.7$ & $83.2 \pm 4.1$ & $+0.7$  & $[-2.7, +3.8]$   & 7/1/4 \\
Push-T        & $92.7 \pm 3.4$ & $89.8 \pm 4.8$ & $+2.8$  & $[+0.8, +4.7]$   & 8/1/3 \\
OGBench-Cube  & $94.7 \pm 1.9$ & $66.2 \pm 5.7$ & $+28.5$ & $[+25.3, +31.3]$ & 12/0/0 \\
\bottomrule
\end{tabular}
\end{table}

\subsection{The gate routes without per-task tuning}
\label{sec:exp-routing}

The reliability gate is computed once per cell from the diagnostic pre-pass and
held fixed. Under the frozen thresholds $(\Tinv, \Tlag) = (0.05, 0.3)$ it routes
all $24$ diagnostic cells (4 tasks $\times$ $ds \in \{3,5,7,9,11,13\}$ at
$ss{=}42$) to the pre-specified per-task recipe, with no cell straddling a
boundary: Two-Room to intuition-dominant, Reacher to the forward-only fallback,
and Push-T and OGBench-Cube to the trusted hybrid
(Figure~\ref{fig:routing-plane}; full per-task diagnostic ranges in
Appendix~\ref{app:extra-results}, Table~\ref{tab:routing}). Because the recipe
is selected automatically from the diagnostics rather than hand-set per task,
the headline gains are not the product of per-task tuning. A one-axis
perturbation audit on the 48 headline cells leaves all routing unchanged for
$|\delta\Tinv| < 0.0376$ and $|\delta\Tlag| < 0.3136$ (recipe-active cells); the
tight boundary is $\Tinv$ on Two-Room and Reacher
(Appendix~\ref{app:gate-margins}).

\begin{figure}[t]
\centering
\includegraphics[width=0.88\textwidth]{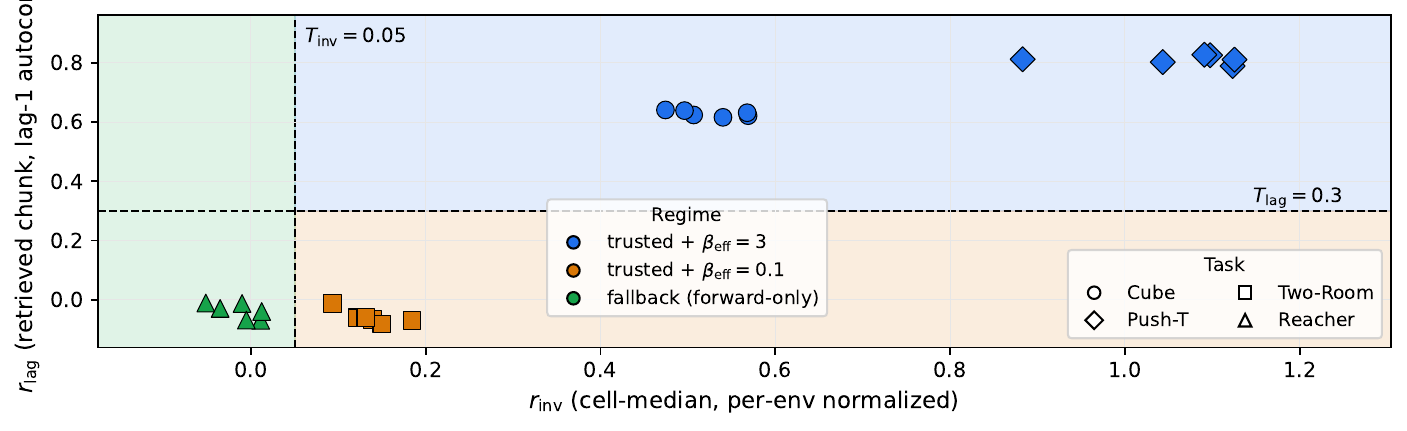}
\caption{Two-feature routing plane $(\rinv, \rlag)$: each point is one of the 24
diagnostic cells, colored by its routed recipe. Frozen
thresholds $(\Tinv, \Tlag) = (0.05, 0.3)$ are the axis-aligned partition lines;
the gate is the \emph{evaluated} routing mechanism, not an interpolated switch.}
\label{fig:routing-plane}
\end{figure}

\subsection{Ablation study}
\label{sec:exp-diag-intuition-elite}

We ablate the two component axes the gate controls (Retrieval Initialization
and the Hybrid Cost), and then ablate the gate itself.

\paragraph{Retrieval Initialization $\times$ Hybrid Cost.}
We cross the two design axes into a $2 \times 3$ grid
(Table~\ref{tab:ablation-grid}): Retrieval Initialization is either off (the
proposal centers on the zero chunk) or on, and the Hybrid Cost uses the
intuition term only ($\alphainv{=}1, \beff{=}0$), the world-model term only
($\alphainv{=}0, \beff{=}1$), or both. The world-model-only, retrieval-off cell
is the headline baseline, and the both-terms, retrieval-on cell is full IMWM;
when Retrieval Initialization is on or the Hybrid Cost uses both terms, the
Reliability Gate is the mechanism that selects the configuration per cell.
Two readings stand out. First, \emph{Retrieval Initialization is load-bearing on
the search-limited tasks}: with the world-model cost alone, turning retrieval on
lifts OGBench-Cube from $66.2$ to $88.3$ and Two-Room from $87.7$ to $98.2$,
matching the diagnosis that the bottleneck is proposal coverage
(\S\ref{sec:diagnosis}); on Reacher, whose world-model baseline is already
strong, retrieval instead hurts ($83.2 \to 64.8$), which is exactly why the gate
disables it there. Second, \emph{the intuition score is not a standalone
replacement}: the intuition-only cost trails full IMWM by a wide margin on
Reacher, Push-T, and OGBench-Cube, and the full configuration (retrieval on,
both terms) is best on every task.

\begin{table}[t]
\centering
\small
\caption{\textbf{Retrieval Initialization $\times$ Hybrid Cost} ($2 \times 3$
grid; mean cell success, \%, over the 12-cell-per-task grid). This grid is
\emph{gate-faithful} rather than a strict orthogonal factorial: the
\emph{both} weights are the gate's per-task choice and the \emph{(on, both)}
row, marked $^\dagger$, is the deployed full IMWM, whose recipe is gate-selected
per task. For Reacher the gate disables retrieval and selects
the forward-only recipe, so Reacher's \emph{both} entries reduce to its
world-model-only fallback and its full-IMWM value ($83.8$) is statistically tied
with the \emph{(off, world-model-only)} baseline ($83.2$) up to run-to-run CEM
variation. The \emph{(off, world-model-only)} row is the headline baseline (c1,
retrieval off); the \emph{(on, world-model-only)} row adds Retrieval
Initialization to the world-model cost (c2). Per-cell paired CIs in
Appendix~\ref{app:extra-results}, Table~\ref{tab:diag2-full}.}
\label{tab:ablation-grid}
\begin{tabular}{ll cccc}
\toprule
Retrieval & Hybrid Cost & Two-Room & Reacher & Push-T & OGBench-Cube \\
\midrule
off & intuition-only   & $96.2$ & $63.0$ & $54.2$ & $73.5$ \\
off & world-model-only & $87.7$ & $83.2$ & $89.8$ & $66.2$ \\
off & both             & $98.8$ & $83.2$ & $92.3$ & $76.8$ \\
\midrule
on  & intuition-only   & $98.3$ & $53.2$ & $60.0$ & $83.8$ \\
on  & world-model-only & $98.2$ & $64.8$ & $91.5$ & $88.3$ \\
on$^\dagger$  & both$^\dagger$ & $\mathbf{99.2}$ & $\mathbf{83.8}$ & $\mathbf{92.7}$ & $\mathbf{94.7}$ \\
\bottomrule
\end{tabular}
\end{table}

\paragraph{Reliability Gate.}
The gate routes per cell to one of three recipes; to show this adaptivity is
necessary we instead \emph{fix} each recipe and evaluate it on all four tasks,
as in a hyperparameter grid search (Table~\ref{tab:ablation-gate}). No single
fixed recipe is best everywhere: the forward-only fallback $(0,1,\text{off})$ is
best on Reacher but collapses on OGBench-Cube; the trusted hybrid
$(1,3,\text{on})$ is best on OGBench-Cube and Push-T; and the
intuition-dominant recipe $(1,0.1,\text{on})$ is best on Two-Room. The
appropriate configuration differs across tasks, and the Reliability Gate selects
each task's best fixed recipe automatically from the diagnostic pre-pass,
matching or statistically tying the per-task best without per-task tuning (on
Reacher the gate's $83.8$ and the fixed forward-only $83.2$ are tied up to
run-to-run CEM variation under the shared fallback recipe).

\begin{table}[t]
\centering
\small
\caption{\textbf{Reliability-Gate ablation} (mean cell success, \%): each fixed
recipe $(\alphainv, \beff, \text{retrieval})$ evaluated on all four tasks versus
the adaptive gate. \textbf{Bold} marks the column maximum, so the co-best fixed
recipes are bolded alongside the adaptive gate; on Reacher the adaptive gate is
the column maximum alone (the fixed forward-only $83.2$ is a near-tie with the
gate's $83.8$ up to run-to-run CEM variation). The best fixed recipe differs
across columns, and the adaptive gate matches or statistically ties each
column's best.}
\label{tab:ablation-gate}
\begin{tabular}{l cccc}
\toprule
Configuration & Two-Room & Reacher & Push-T & OGBench-Cube \\
\midrule
Fixed $(1,3,\text{on})$ \,(trusted hybrid)       & $99.0$ & $74.3$ & $\mathbf{92.7}$ & $\mathbf{94.7}$ \\
Fixed $(1,0.1,\text{on})$ \,(intuition-dominant)  & $\mathbf{99.2}$ & $65.5$ & $71.7$ & $91.5$ \\
Fixed $(0,1,\text{off})$ \,(forward-only)         & $87.7$ & $83.2$ & $89.8$ & $66.2$ \\
\midrule
Reliability Gate (adaptive)                       & $\mathbf{99.2}$ & $\mathbf{83.8}$ & $\mathbf{92.7}$ & $\mathbf{94.7}$ \\
\bottomrule
\end{tabular}
\end{table}

\subsection{Additional Experiments}
\label{sec:exp-ablations}

Two further checks support the design; both are detailed in the appendices.
\emph{(i)} A CEM-budget sweep ($T \in \{15,30,60\}$ at
$S{=}300$) shows IMWM's gains persist through $T{=}60$ on the Two-Room and
OGBench-Cube sentinels and do not close with more iterations
(Appendix~\ref{app:extra-results}, \S\ref{sec:exp-add-budget}). \emph{(ii)}
Substituting a heteroscedastic forward-density likelihood for the rollout-MSE
term regresses planner ranking despite better validation NLL
(Appendix~\ref{app:audit-ledger}), consistent with objective
mismatch~\citep{lambert20_objective_mismatch}. Finally, on the oracle-dynamics
diagnostic of \S\ref{sec:diagnosis-oracle}, whenever the CEM population contains
a goal-reaching candidate the terminal latent-MSE ranks one first (rank-$0$ in
$\geq 95.6\%$ of such replans on the validated tasks;
Appendix~\ref{app:extra-results}, \S\ref{sec:exp-diag-rank-success}),
corroborating that the bottleneck is candidate coverage, not ranking.

\section{Limitations and future directions}
\label{sec:limitations}

\paragraph{Limitations.}
\label{sec:limitations-current}
IMWM's reliability gate is hand-designed and frozen (two interpretable
diagnostics and fixed thresholds rather than a learned router), so it depends
on those features proxying intuition reliability and on the thresholds suiting
the evaluated distribution; a task whose reliability axis these features miss
would need recalibration. Our results are on four pixel-based goal-reaching
tasks from a single benchmark family, and we make no claim of transfer to
broader benchmarks or to real-robot control. Both the intuition scorer and the
retrieval bank are built from successful demonstrations, so IMWM cannot bootstrap
the intuition signal from interaction alone. The Push-T gain is smaller
($+2.8$~pp) because the baseline is already strong there, and Push-T is
separately excluded from the oracle anchor for a simulator-physics confound
(Appendix~\ref{app:pusht-quarantine}); the two caveats are independent. We also
leave to future work the per-task trajectory and latent-elite visualizations
(the current pipeline does not log candidate latents) and the full $(S, T)$
CEM-budget grid that would directly probe the proposal-volume axis of
\S\ref{sec:diagnosis}.

\paragraph{Future directions.}
\label{sec:limitations-future}
The most natural extension is a learnable, amortized gate that predicts the
recipe from a richer reliability feature pool, removing human-set thresholds;
the continuous relaxation of Appendix~\ref{app:conc-mode-defs} is a step toward
a smoothly interpolating deployed gate. Beyond the gate, we see value in
multi-step or multi-modal intuition models; broader benchmarks and real-robot
transfer~\citep{nagabandi20pddm,nair22r3m}; reducing the demonstration
requirement through self-supervised inverse-dynamics
objectives~\citep{pathak17icm} or self-collected banks; and joint,
planning-aware fine-tuning of the intuition score and world
model~\citep{farahmand17vaml,tdmpc}, potentially with a value-guided cost
term~\citep{eysenbach22contrastiverl} that the gate could arbitrate alongside
the existing two. Finally, our experiments cover four tasks from a single
benchmark family; a broader evaluation across the wider environment suite of
the \texttt{stable-worldmodel} platform~\citep{swm}, on which IMWM is built, is
a natural next step for testing whether the gate's reliability diagnostics
generalize.

\section{Conclusion}
\label{sec:conclusion}

Latent world-model planning with a finite CEM budget can be limited by its
\emph{search} rather than by its forward predictor: even an idealized world
model (the learned predictor replaced by literal environment rollout under
the original budget) still fails on some tasks because the planner's finite
queries never place a goal-reaching candidate in the population
(\S\ref{sec:diagnosis-oracle}). We address this not by training a better
predictor but by pairing the world model with an \emph{intuition model}, a
jointly trained inverse-side encoder and bilinear scorer. The intuition model is
consumed through three components: Retrieval Initialization centers the
CEM proposal on a retrieved demonstration chunk, a Hybrid Cost combines
the intuition score with the world-model rollout cost, and a Reliability
Gate decides per cell, before any query is spent, how much to rely on the
intuition model.
The resulting planner, \emph{IMWM}, improves over the world-model-only baseline
to $99.2\%$ on Two-Room ($+11.5$~pp, $12/0/0$), $94.7\%$ on OGBench-Cube
($+28.5$~pp, $12/0/0$), $92.7\%$ on Push-T ($+2.8$~pp, $8/1/3$), and $83.8\%$ on
Reacher (a near-tie, $+0.7$~pp, $7/1/4$) on a 48-cell fresh-seed grid. The
algorithmic family is fixed across all experiments; the gate's only per-cell
decision is which of three pre-specified recipes to use, so the gains do not
come from per-task tuning. Our main limitations
(\S\ref{sec:limitations-current}) are a hand-designed gate, a single benchmark
family, and a demonstration-data dependence; these point to learnable
arbitration, broader transfer, and joint planning-aware fine-tuning as the
immediate next steps (\S\ref{sec:limitations-future}).

\bibliography{refs}
\bibliographystyle{iclr2026_conference}

\appendix
\section{Formal limitations and proposal-side escape}
\label{app:formal}

The arguments in Section~\ref{sec:diagnosis} are formalized here. \textbf{Notation.} $c_m$ denotes the Lebesgue volume of the unit ball in $\mathbb{R}^m$. $B_\varepsilon(x) = \{a \in \mathbb{R}^m : \|a - x\|_2 \le \varepsilon\}$. $\chi^2_m$ denotes a chi-square random variable with $m$ degrees of freedom. The action space is $A = [0,1]^m \subset \mathbb{R}^m$.

\subsection{Finite-query proposal-volume bottleneck under exact world model}
\label{app:theorem-a1}

\paragraph{Setup.}
Fix $m \ge 1$ and $\varepsilon > 0$. Pick a target $a^\dagger \in A$. Define the success set $S_\varepsilon(a^\dagger) := B_\varepsilon(a^\dagger) \cap A$. The MDP $\mathcal{M}_\varepsilon$ has three states $\{s_{\mathrm{start}}, s_{\mathrm{goal}}, s_{\mathrm{fail}}\}$, action space $A$, deterministic dynamics from $s_{\mathrm{start}}$ ($\to s_{\mathrm{goal}}$ if $a \in S_\varepsilon(a^\dagger)$, else $\to s_{\mathrm{fail}}$), absorbing terminal states, and reward $R(s_{\mathrm{goal}}) = 1$, $R(s_{\mathrm{fail}}) = 0$. The optimal value $V^*(s_{\mathrm{start}}) = 1$.

\paragraph{Oracle-strength LeWM cost surface.}
We do not claim that a specific $(\phi, \hat M)$ minimizes the LeWM loss on this MDP; for a finite three-state latent, the prediction-loss/SIGReg objective need not be reached by any standard training. Instead we consider any $(\tilde\phi, \tilde{\hat M})$ that \emph{induces} the $L^2$-to-goal cost surface $J_{\mathrm{LeWM}}(a) \in \{0, c_{\mathrm{fail}}\}$ with $0$ exactly on $S_\varepsilon(a^\dagger)$ and $c_{\mathrm{fail}} > 0$ elsewhere. This is the LeWM framework's strongest possible cost surface: value-aligned at every action. The bound below depends only on the existence of this binary-oracle cost surface, not on how $(\tilde\phi, \tilde{\hat M})$ was obtained.

\begin{lemma}[Finite-query identification under black-box binary feedback]
\label{lem:identification}
Let $P$ be any randomized planner that interacts with a hidden cost function $f : A \to \{0, c_{\mathrm{fail}}\}$ \emph{only} by issuing at most $n$ adaptive queries $x_1, \ldots, x_n \in A$, each receiving the scalar feedback $f(x_i)$, and then outputting a final action $x_{\mathrm{out}} \in A$. The planner may use the entire transcript of past queries plus feedbacks and its internal randomness to choose each $x_i$ and $x_{\mathrm{out}}$, but cannot access $f$ other than through queries.

Let $a^\dagger \sim \mathrm{Uniform}(\mathrm{int}(A))$ and $f(a) = 0$ iff $a \in S_\varepsilon(a^\dagger)$, else $f(a) = c_{\mathrm{fail}} > 0$. Then for any $\varepsilon > 0$:
\begin{equation*}
\Pr_{a^\dagger, P}\!\left[ \exists\, i \in \{1, \ldots, n, \mathrm{out}\} : x_i \in S_\varepsilon(a^\dagger) \right] \le (n + 1)\, c_m\, \varepsilon^m.
\end{equation*}
\end{lemma}

\begin{proof}
Fix the planner's internal randomness $\omega$. By the black-box-access assumption, until any query produces feedback $0$, the planner's feedback transcript is the constant string $(c_{\mathrm{fail}}, \ldots, c_{\mathrm{fail}})$, and the planner's query sequence on this transcript is deterministic in $\omega$ only: write it $x_1^0(\omega), \ldots, x_n^0(\omega), x_{\mathrm{out}}^0(\omega)$. A success at any query $i$ requires $a^\dagger \in B_\varepsilon(x_i^0(\omega))$. Using $\mathrm{Vol}(B_\varepsilon(x) \cap A) \le c_m \varepsilon^m$ for any $x \in A$:
\begin{equation*}
\Pr_{a^\dagger}\!\left[ \exists\, i : a^\dagger \in B_\varepsilon(x_i^0(\omega)) \mid \omega \right]
\;\le\; \mathrm{Vol}\!\left( \textstyle\bigcup_{i=0}^{n} B_\varepsilon(x_i^0(\omega)) \cap A \right)
\;\le\; (n + 1)\, c_m\, \varepsilon^m.
\end{equation*}
Integrating over $\omega$ gives the lemma. By averaging over $a^\dagger$, there exists a fixed $a^\dagger \in A$ with $\Pr_P[\mathrm{success} \mid a^\dagger] \le (n + 1)\, c_m\, \varepsilon^m$.
\end{proof}

The black-box-access assumption is essential: an optimizer with symbolic access to $f$ could read off $a^\dagger$ directly. The lemma bounds finite-query black-box planners using the world model as a cost oracle, which is precisely the LeWM+CEM operating regime.

\begin{theorem}[Fixed-budget LeWM/CEM proposal-volume bottleneck]
\label{thm:bottleneck}
Let LeWM+CEM be a randomized planner that accesses the world model only through at most $n := N T$ cost queries ($N$ samples per iteration over $T$ iterations) and one final output action. For every such budget $(N, T)$ and every $\delta \in (0, 1)$, there exist $\varepsilon > 0$ and a target $a^\dagger \in A$ such that on $\mathcal{M}_\varepsilon$ with the oracle-strength LeWM cost surface above,
\begin{equation*}
\Pr\!\left[ \mathrm{LeWM\text{+}CEM\ output} \in S_\varepsilon(a^\dagger) \right] \;\le\; \delta.
\end{equation*}
\end{theorem}

\begin{proof}
The cost surface $J_{\mathrm{LeWM}}(a) \in \{0, c_{\mathrm{fail}}\}$ satisfies the binary-feedback hypothesis of Lemma~\ref{lem:identification}. CEM is a randomized finite-query black-box planner: it issues up to $N T$ queries and outputs the elite mean (or lowest-cost candidate). Applying Lemma~\ref{lem:identification} with $n = N T$ and choosing $\varepsilon < (\delta / ((NT + 1) c_m))^{1/m}$:
\begin{equation*}
\Pr_{a^\dagger, \mathrm{CEM}}[\mathrm{success} \mid a^\dagger \sim \mathrm{Uniform}] \;\le\; (NT + 1)\, c_m\, \varepsilon^m \;\le\; \delta.
\end{equation*}
By averaging there exists a fixed $a^\dagger$ achieving the bound.
\end{proof}

The bound $(NT + 1)\, c_m\, \varepsilon^m$ is independent of $(\phi, \hat M)$ quality; improving the world model under the constructed oracle-cost regime does \emph{not} shrink this bound. The construction makes the oracle cost surface maximally informative under the binary-feedback assumption, so the bottleneck lives at the planner's finite-query budget and the proposal's inability to inject information about $a^\dagger$ into the query distribution. The bound is existential / worst-case: for any fixed budget, there exist hard manifolds on which the bound is tight.

\begin{figure}[t]
\centering
\includegraphics[width=0.95\textwidth]{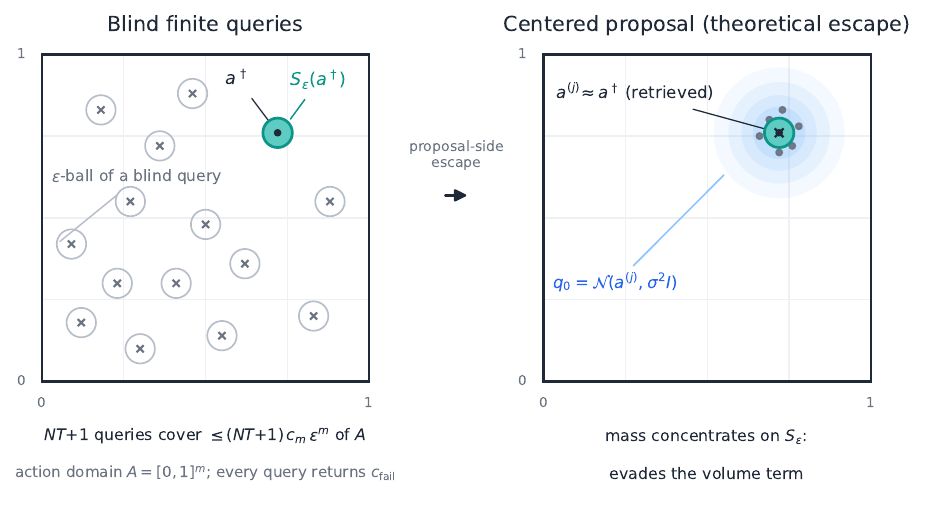}
\caption{\textbf{The search bottleneck of Theorem~\ref{thm:bottleneck}.}
\emph{Left:} under a perfect, value-aligned cost, a finite planner's $NT{+}1$
black-box queries are blind (each returns the constant $c_{\mathrm{fail}}$ until
one lands in the tiny success ball $S_\varepsilon(a^\dagger)$); the union of
their $\varepsilon$-balls covers at most $(NT{+}1)\,c_m\,\varepsilon^m$ of the
action domain, independent of predictor quality. \emph{Right:} a proposal
centered near $a^\dagger$ concentrates mass on $S_\varepsilon$ and evades the
volume term. This is the \emph{theoretical} proposal-side escape
(Appendix~\ref{app:theorem-a4}, Theorem~\ref{thm:proposal-escape}; one-step /
best-so-far CEM under retrieval-coverage assumptions) that motivates the
Retrieval Initialization component of \S\ref{sec:method}; we do not claim
retrieval initialization is independently responsible for the empirical gains
(Appendix~\ref{app:extra-results}).}
\label{fig:bottleneck}
\end{figure}

\subsection{Proposal-side escape via retrieval-initialized CEM}
\label{app:theorem-a4}

\paragraph{Setup.}
IMWM adds a fixed demonstration bank $R = \{(k_0^{(i)}, k_g^{(i)}, a^{(i)})\}_{i=1}^{|R|}$ whose entries store cosine-normalized start/goal latents $(k_0^{(i)}, k_g^{(i)})$ produced by the frozen intuition encoder $\phi^I$ together with the demonstration action chunks $a^{(i)}$. The deployed IMWM retrieval criterion is cosine nearest-neighbor lookup in the intuition encoder's latent space: given an eval query $(o_0, o_g)$ with $(\zinv{0}, \zinv{g}) = (\phi^I(o_0), \phi^I(o_g))$, the planner selects $a^{(j)} = a^{(i^*)}$ where $i^* = \arg\max_i \mathrm{cos}([\zinv{0}; \zinv{g}], [k_0^{(i)}; k_g^{(i)}])$.

\paragraph{Conditions.}
\begin{itemize}
\item \textbf{(C1) Bank coverage}: the bank $R$ contains at least one entry whose action chunk lies within $\varepsilon/2$ of a successful action $a^\dagger$ for the query.
\item \textbf{(C2) Encoder-geometry calibration}: the cosine nearest-neighbor lookup in $z^I$-space selects such a covered entry; equivalently, $\|a^{(j)} - a^\dagger\|_2 \le \varepsilon/2$.
\item \textbf{(C3) Proposal gate engages}: the retrieval-initialized sampling mode is active for this cell. The Section~\ref{sec:method-gate} cost-side weight $\alphainv$ is empirically correlated with the proposal-gate state but is \emph{not} the mathematical source of the volume break; the volume reduction comes from changing the proposal distribution.
\end{itemize}

Conditions (C1) and (C2) separate the bank-coverage property (a property of $R$) from the encoder-geometry property (a property of $\phi^I$): coverage is about whether the bank contains a useful action chunk; calibration is about whether the cosine lookup in the intuition encoder's latent space selects it.

The IMWM family contains three planner variants admitting closed-form lower bounds.

\begin{theorem}[Variant 1: retrieval-only]
\label{thm:retrieval-only}
Under (C1), (C2), (C3), if the planner directly executes the retrieved action $a^{(j)}$, then the executed action lies in $B_\varepsilon(a^\dagger)$ with probability $1$.
\end{theorem}

\begin{proof}
By (C2), $a^{(j)} \in B_{\varepsilon/2}(a^\dagger) \subset B_\varepsilon(a^\dagger)$. The planner outputs $a^{(j)}$ directly.
\end{proof}

\begin{theorem}[Variant 2: best-observed-candidate]
\label{thm:best-observed}
Under (C1), (C2), (C3), if the planner force-inserts $a^{(j)}$ into its candidate set \emph{and} returns $\arg\min_i J_{\mathrm{LeWM}}(x_i)$, then the returned action lies in $B_\varepsilon(a^\dagger)$ with probability $1$.
\end{theorem}

\begin{proof}
$J_{\mathrm{LeWM}}(a^{(j)}) = 0$ (by (C2) and the Appendix~\ref{app:theorem-a1} cost surface). Any candidate $x \notin B_\varepsilon(a^\dagger)$ has cost $c_{\mathrm{fail}} > 0$. Thus $a^{(j)}$ strictly dominates any failure-region candidate, and the best-observed-candidate output lies in $B_\varepsilon(a^\dagger)$.
\end{proof}

\begin{theorem}[Variant 3: one-step elite-mean CEM proposal-side escape]
\label{thm:proposal-escape}
Assume the planner performs \emph{one iteration} of elite-mean CEM with retrieval-initialized proposal: it draws $N$ candidates with a forced-mean candidate at $a^{(j)}$ and $N - 1$ IID samples from $q_0 = \mathcal{N}(a^{(j)}, \sigma^2 I)$ truncated to $A$. Assume the away-from-boundary condition $B_{\varepsilon/2}(a^{(j)}) \subset A$. The planner selects the top-$K$ elites by cost and returns the elite mean $\mu_{\mathrm{out}}$. Then under (C1), (C2), (C3),
\begin{equation*}
\Pr\!\left[\mu_{\mathrm{out}} \in B_\varepsilon(a^\dagger)\right] \;\ge\; \Pr\!\left[1 + \mathrm{Binomial}(N - 1, p) \ge K\right] \quad\text{where}\quad p := \Pr\!\left[\chi^2_m \le (\varepsilon/(2\sigma))^2\right].
\end{equation*}
\end{theorem}

\begin{proof}
By (C2), $a^{(j)} \in B_{\varepsilon/2}(a^\dagger) \subset B_\varepsilon(a^\dagger)$, so the forced-mean candidate has cost $0$ deterministically. For each of the $N - 1$ IID samples $a$ from $q_0$ truncated to $A$, the away-from-boundary condition gives $B_{\varepsilon/2}(a^{(j)}) \subset A$, so
\begin{equation*}
\Pr_{a \sim q_0}\!\left[a \in B_{\varepsilon/2}(a^{(j)})\right] \;\ge\; \Pr_{a \sim \mathcal{N}(a^{(j)}, \sigma^2 I)}\!\left[a \in B_{\varepsilon/2}(a^{(j)})\right] \;=\; \Pr\!\left[\chi^2_m \le (\varepsilon/(2\sigma))^2\right] \;=\; p.
\end{equation*}
By the triangle inequality, any sample in $B_{\varepsilon/2}(a^{(j)})$ lies in $B_\varepsilon(a^\dagger)$ and has cost $0$. Let $S$ be the count of IID candidates in $B_{\varepsilon/2}(a^{(j)})$; then $S$ stochastically dominates $\mathrm{Binomial}(N - 1, p)$. Conditional on $\{1 + S \ge K\}$, at least $K$ candidates have cost $0$, strictly dominating any failure-region candidate; every top-$K$ elite is from the zero-cost set, which is contained in $B_\varepsilon(a^\dagger)$. By convexity of $B_\varepsilon(a^\dagger)$, the elite mean $\mu_{\mathrm{out}}$ lies in $B_\varepsilon(a^\dagger)$.
\end{proof}

The bound $\Pr[1 + \mathrm{Binomial}(N - 1, p) \ge K]$ tends to $1$ as $\sigma/\varepsilon \to 0$. For $K = 1$ it equals $1$ deterministically. For $K > 1$, a union bound gives $\Pr[\mathrm{failure}] \le 2^{N-1} (1 - p)^{N - K + 1}$. If the away-from-boundary condition does not hold, the bound goes through with $p$ redefined as the actual mass of the truncated proposal on $B_{\varepsilon/2}(a^{(j)})$; the structural argument ($\ge K$ zero-cost candidates $\Rightarrow$ elite mean in $B_\varepsilon(a^\dagger)$ by convexity) is unchanged.

\paragraph{One-step scope.}
The deployed multi-step elite-mean solver updates the proposal across iterations and returns the \emph{final} elite mean. A success event at iteration 1 ($\mu_1 \in B_\varepsilon(a^\dagger)$) does not directly imply success at iteration $T$ without best-so-far retention. Theorem~\ref{thm:proposal-escape} is stated for one-step elite-mean CEM; the same bound holds for multi-step CEM with best-so-far retention. The deployed multi-step run is empirically consistent with the one-step mechanism on the OGBench-Cube and Two-Room cells in Section~\ref{sec:experiments}, but we do not have a formal multi-step lower bound.

\begin{theorem}[Scope: IMWM reduces to LeWM under retrieval-informativeness failure with (C4) full fallback calibration]
\label{thm:fallback-reduce}
Add an additional condition:
\begin{itemize}
\item \textbf{(C4) Full fallback calibration} (\emph{calibration assumption, not a consequence of (C1) or (C2) failure}): when the bank-coverage condition (C1) fails or the scorer-selection condition (C2) fails (i.e., $\|a^{(j)} - a^\dagger\| > \varepsilon/2$), the evaluated cell-level gate of Section~\ref{sec:method-gate} satisfies $\rinv(c) \le \Tinv$, so $\mathrm{recipe}(c) = (\alphainv{=}0, \beff{=}1, \text{retrieval off})$. Thus the proposal-gate state disables retrieval-init and the cost-side intuition contribution is disabled; equivalently, the planner's black-box feedback is the LeWM forward cost surface up to a monotone transform.
\end{itemize}
Under (C1) or (C2) failure together with (C4), Theorem~\ref{thm:bottleneck} applies unchanged: the planner's success probability is bounded above by $(NT + 1)\, c_m\, \varepsilon^m$.
\end{theorem}

\paragraph{Note on (C4).}
(C4) is a \emph{calibration assumption} on the trained $\Dinv$ together with the Section~\ref{sec:method-gate} thresholds $(\Tinv, \Tlag)$, not a mathematical consequence of (C1) or (C2) failing. The implication ``returned $a^{(j)}$ fails (C2) $\Rightarrow \rinv(c) \le \Tinv$'' is an empirical calibration property of the evaluated hard gate. A miscalibrated $\Dinv$ could engage retrieval despite uninformative retrieval (violating the routing step of (C4)) or vary $\Dinv$ across failure-region actions (injecting side information that breaks Lemma~\ref{lem:identification}'s binary-feedback hypothesis). Without (C4), the volume-bound argument may fail and IMWM could underperform LeWM by injecting a misleading proposal. The analogous continuous-relaxation condition would be $\concc < \Ttrust$, but that is the Appendix~\ref{app:conc-mode-defs} parameterization analysis rather than the evaluated gate. The Reacher near-tie in Section~\ref{sec:experiments} (paired 7/1/4, mean $\Delta = +0.7$~pp) is consistent with this calibration property; we do not claim it holds universally.

\subsection{Conditional representation-level limitation}
\label{app:r2a}

A separate, conditional limitation exists at the \emph{representation} layer rather than the planner layer. We record it as a remark, not a theorem, because it does not follow from the LeWM training objective alone.

\begin{remark}[Conditional representation-level limitation]
\label{rem:r2a}
Under the additional assumption that the encoder $\phi_\theta$ is constrained to be \emph{minimally sufficient} for next-state prediction, for example via an explicit information-bottleneck penalty, a hard dimensional bottleneck $d_z \ll d_o$, or a minimum-norm selection rule among prediction-optimal encoders, the encoder can discard observation features that are reward-relevant but transition-irrelevant. On goal-reaching MDPs where the goal-distinguishing feature is transition-irrelevant for next-state prediction, the $L^2$-to-goal cost cannot separate $\phi(s_{\mathrm{goal}})$ from $\phi(s_{\mathrm{decoy}})$, and planning fails on this distinction. This is the value-aware-model-learning limitation in the lineage of \citet{farahmand17vaml}.
\end{remark}

\paragraph{Conditionality.}
The LeWM training objective ($L^2$ next-embedding prediction $+$ SIGReg) does not, by itself, force minimal sufficiency. The prediction loss requires the latent to retain information useful for predicting the next embedding under each action, but it does not penalize redundant observation information that is harmless for prediction. SIGReg constrains the \emph{shape} of the marginal distribution of $\phi_\theta(o)$ via Epps--Pulley statistics on random unit-vector projections, not its \emph{semantic content}. In particular, a nuisance feature can be retained if it does not hurt next-embedding prediction and can be encoded within a SIGReg-compatible marginal. The representation-level limitation is therefore a real concern in the minimal-sufficiency regime but cannot be claimed as an unconditional consequence of the LeWM loss alone.

\paragraph{Orthogonality to the proposal-volume bottleneck.}
Remark~\ref{rem:r2a} addresses a \emph{representation} failure: even if the planner is exact, the cost surface in latent does not encode all task-relevant distinctions. Theorem~\ref{thm:bottleneck} (Appendix~\ref{app:theorem-a1}) addresses a \emph{search} failure: even if the cost surface is perfectly value-aligned, the planner cannot find the success manifold by volume search. These are orthogonal limitations. IMWM's proposal-side escape on the search-failure axis does \emph{not} address the representation axis: $\Dinv$ is itself conditioned on $(\phi(o_0), \phi(o_g))$, so if the encoder has discarded a goal-distinguishing feature, neither LeWM nor IMWM can recover the distinction. Closing the representation-level gap requires a different axis of improvement, such as reward-aware encoder training or a planning cost that is not $L^2$-to-goal in the prediction-trained latent.

\section{Energy-based view and demo-rank saturation audit}
\label{app:energy-audit}

This appendix complements the diagnosis of Section~\ref{sec:diagnosis} and the formal analysis of Appendix~\ref{app:formal} by recording the empirical and mechanistic findings that motivate preserving IMWM's hybrid cost. None of the material here introduces new methodology; it consolidates a multi-step diagnostic and ablation campaign that ran after the Section~\ref{sec:experiments} grand-total table was frozen.

\subsection{Finite-candidate Gibbs interpretation}
\label{app:gibbs}

For any iteration of the deployed CEM solver, which samples $S = 300$ candidates $\{a_1, \ldots, a_S\}$ at each iteration (\S\ref{sec:method-pseudocode}, Appendix~\ref{app:cem-hyperparams}), the per-anchor $z$-scored composite cost $J(a_k; e, t)$ induces a softmax/Gibbs distribution over the finite candidate set:
\begin{equation*}
p_S(k \mid e, t) = \frac{\exp\!\big(-J(a_k; e, t) / T\big)}{\sum_j \exp\!\big(-J(a_j; e, t) / T\big)}, \qquad T > 0 \text{ interpretive.}
\end{equation*}
This is well-defined as a categorical distribution over $S$ candidates for any finite $J$. A continuous Boltzmann interpretation over the solver's continuous action space is heuristic only: the deployed solver does not globally enforce a hard clip on its isotropic-Gaussian candidate samples, so we do not claim continuous normalizability. The temperature $T$ is interpretive (not tuned); the top-$K = 30$ elite fraction (a 10\% elite fraction at $S = 300$) corresponds heuristically to a low-temperature target but is not an explicit temperature.

\subsection{CEM as adaptive cross-entropy / maximum-likelihood projection}
\label{app:cem-ce}

CEM is not literally an EM algorithm. The deployed CEM loop fits a parametric Gaussian proposal $q_\psi$ to a low-energy elite set by iteratively
\begin{enumerate}
\item sampling $S = 300$ candidates from $q_\psi^{(k-1)}$;
\item defining the empirical elite/level-set target distribution over the top-$K = 30$ candidates by $J$;
\item refitting the parametric proposal $q_\psi^{(k)}$ by maximum likelihood (equivalently, $q_\psi^{(k)} = \arg\min \mathrm{KL}(\delta_{\text{elites}} \| q_\psi)$) onto the elite candidates.
\end{enumerate}
The optimizer returns the proposal mean $\mu^{(K)}$ as its action estimate, which need not equal the global MAP of $p_S(\cdot \mid e, t)$. The planner is thus a retrieval-initialized \emph{approximate} MAP search: the retrieval-top-1 mean $\abase$ and the per-cell weights $(\alphainv, \beff)$ are amortized from demonstration/retrieval statistics, while CEM performs per-query iterative search over the energy. This complements Appendix~\ref{app:formal}: it analyzes a worst-case finite-query bottleneck under an idealized sharp success manifold, while this appendix describes the deployed CEM loop as an adaptive cross-entropy optimizer over the empirical energy. The two perspectives are complementary rather than competing claims.

\subsection{Zero-training demo-rank audit}
\label{app:demo-rank}

We measure the rank of the demonstration action under $J$ on held-out training-distribution anchors. The audit shares IMWM's bank (cosine top-1 retrieval in the intuition encoder's latent space) and differs from the deployed IMWM only in the candidate-noise distribution: the audit uses 256 AR(1) ($\rho{=}0.9$, $\sigma{=}1.0$) chunks around the cosine top-1 retrieved mean, while IMWM uses isotropic Gaussian ($\sigma{=}1.0$) chunks. We did not re-run this offline audit under the isotropic-Gaussian candidate geometry; the saturation finding below is interpretive evidence about the hybrid cost $J$ as a demonstration-vs-noise discriminator. For each of $n_{\text{held-out}} = 100$ anchors per task (episode-level validation split disjoint from the candidate-bank training episodes), the audit candidate set is 257 chunks: 1 demonstration action chunk (from the source episode) at slot 0, plus 256 AR(1) ($\rho{=}0.9$, $\sigma{=}1.0$) candidates around the cosine top-1 retrieved mean. The 256 here is the audit choice, distinct from the deployed CEM solver's $S = 300$. We score each candidate under $J$ with fixed task-regime weights $(\alphainv, \beff) \in \{(1, 3), (1, 3), (1, 0.1)\}$ for OGBench-Cube / Push-T / Two-Room, matching the trusted recipes for these tasks (and equal to the saturated continuous-relaxation regimes in Appendix~\ref{app:conc-mode-defs}). The rank of the demonstration action under $J$ appears in Table~\ref{tab:demo-rank}.

\begin{table}[t]
\centering
\small
\caption{Demonstration-action rank distribution under the hybrid cost $J$ on held-out training-distribution anchors. The audit uses IMWM's cosine top-1 retrieval but the candidate noise around the retrieved mean is AR(1) ($\rho{=}0.9$, $\sigma{=}1.0$) rather than IMWM's deployed isotropic Gaussian. Each row reports the cumulative distribution of demo rank within the 257-chunk audit set (1 demo + 256 noise chunks). $N = 100$ anchors per task.}
\label{tab:demo-rank}
\begin{tabular}{lrrrrr}
\toprule
\textbf{Task}    & rank $= 1$ & rank $\le 2$ & rank $\le 5$ & rank $\le 20$ & rank $> 20$ \\
\midrule
OGBench-Cube             & 95.0\% & 97.0\% & 98.0\% & 100.0\% & 0.0\% \\
Push-T           & 84.0\% & 92.0\% & 97.0\% & 100.0\% & 0.0\% \\
Two-Room         & 94.0\% & 96.0\% & 100.0\% & 100.0\% & 0.0\% \\
\bottomrule
\end{tabular}
\end{table}

No anchors on any task have demonstration rank greater than 20 under $J$ on this substrate. The hybrid cost is near-saturated as a demonstration-vs-noise discriminator under the cosine-retrieval-centered AR(1) candidate geometry of this audit (the isotropic-deployment caveat stated above applies).

\subsection{Push-T oracle-diagnostic quarantine (non-load-bearing)}
\label{app:pusht-quarantine}

The two on-disk Push-T oracle diagnostic JSONLs are excluded from
Figure~\ref{fig:diag1-rank-histogram} and the main-text rank-vs-success
claim in \S\ref{sec:exp-diag-rank-success}. We summarize them here as
side-result only.

\textbf{ds3\_ss42 (invalid adapter).} This file was generated under the
pre-arbiter-fix Push-T adapter (Codex 778 audit). The pymunk
arbiter-cache leak invalidates the env-state restore, so the recorded
ranks are not interpretable as native Push-T physics. The cell is
labeled \texttt{Invalid PushT adapter} in
\texttt{oracle\_lewm\_report.md}~\S2.3 and is excluded from the active
result tree.

\textbf{ds3\_ss1 (modified-physics partial cell, killed mid-run).} This
file was generated after the arbiter fix but under
\texttt{collision\_persistence=0} (the contact-stability workaround).
Under this physics modification, the native c1 baseline itself drops
from $94\%$ to $30\%$ (\texttt{oracle\_lewm\_report.md}~\S6.1), so no
comparable native-Push-T baseline exists for the cell. The run was
also terminated at 8h47m after pymunk \texttt{\_post\_solve} assertion
warnings accumulated under sustained-contact CEM rollouts; only $103$
diagnostic records were written to disk and no \texttt{success\_rate}
metric was finalized.

Combined, the two files yield $203$ records with $15.3\%$ has-success
and $71\%$ rank-$0$. Because neither file is a valid native-Push-T
oracle measurement, we do not use these numbers to support the
proposal-coverage diagnosis on Push-T.

\subsection{Search-space audit: six intervention layers, all closed on this substrate}
\label{app:audit-closure}

Following the saturation finding in Appendix~\ref{app:demo-rank}, we tested \textbf{eight interventions across six conceptually distinct layers} above the frozen LeWM encoder $+\, \Dinv$. Each intervention is a pre-specified method with explicit closure criteria (offline gates and / or online smoke). Table~\ref{tab:audit-ledger} summarizes outcomes; no tested intervention improved the retained IMWM recipe on OGBench-Cube / Push-T / Two-Room.

\begin{table}[t]
\centering
\small
\caption{Search-space map of frozen-substrate cost / proposal / retrieval interventions, ordered by intervention layer. All interventions hold the IMWM evaluated recipe constant except at the intervention's stated layer.}
\label{tab:audit-ledger}
\begin{tabular}{p{0.30\textwidth}p{0.25\textwidth}p{0.35\textwidth}}
\toprule
\textbf{Intervention} & \textbf{Layer} & \textbf{Outcome on OGBench-Cube / Push-T / Two-Room} \\
\midrule
Heteroscedastic forward-density head & Cost composition / precision head & Catastrophic regression on OGBench-Cube / Push-T \\
Inverse-scorer retrain with retrieval hard negatives & Inverse-scorer retraining & Neutral / marginal regression on fresh-grid extension \\
Alternative forward-CPC residual cost & Cost composition + precision (auxiliary-predictor residual) & $-12$~pp OGBench-Cube, $-31$~pp Push-T \\
LeWM-residual precision-weighted cost & Cost composition + precision (LeWM residual) & $-2.7$~pp OGBench-Cube, $-4.0$~pp Push-T, $-20.7$~pp Two-Room \\
1-step head replacement of 5-step rollout & Forward-primitive replacement & $-25.3$~pp OGBench-Cube, $-68$~pp Push-T \\
Stochastic-slot proposal augmentation ($q_\phi$ residual) & Stochastic-slot proposal augmentation & $+0.22$~pp aggregate, sign-test $p = 1.000$ \\
Deterministic-slot retrieval reranking & Deterministic-slot retrieval reranking & Closed before online evaluation by center-effect gate (cosine top-1 has post-refit stability that direct-$J$ reranking lacks) \\
Rollout-residual readout adapter & Adapter over rollout-residual readout & Closed before online evaluation by Appendix~\ref{app:demo-rank} demo-rank audit \\
\bottomrule
\end{tabular}
\end{table}

No intervention crosses the six layers' pre-specified closure criteria on these three tasks under the IMWM evaluated recipe; the audit is a closure result on this substrate, not an impossibility result.

\subsection{Mechanism: rollout depth carries the rank signal}
\label{app:rollout-depth}

The rank signal in $\mathrm{MSE}$ derives specifically from the 5-step rollout depth, not from any single horizon. Table~\ref{tab:rollout-mrr} reports the mean reciprocal rank (MRR) of the demonstration action under the partial-horizon residual $\|z_{\text{pred},h}(a) - \zlewm{g}\|^2$ on the same 257-candidate set as Appendix~\ref{app:demo-rank}, for each macro horizon $h = 1, \ldots, 5$.

\begin{table}[t]
\centering
\small
\caption{Demonstration-action MRR under partial-horizon residual at horizons $h = 1, \ldots, 5$, on the same audit substrate as Table~\ref{tab:demo-rank} (cosine top-1 retrieval + AR(1) candidate noise). The rank signal builds steeply across horizons; a 1-step head ($h{=}1$) is essentially uninformative on this candidate geometry.}
\label{tab:rollout-mrr}
\begin{tabular}{rrr}
\toprule
\textbf{Horizon $h$} & \textbf{OGBench-Cube MRR} & \textbf{Push-T MRR} \\
\midrule
1                    & 0.018             & 0.010              \\
2                    & 0.070             & 0.031              \\
3                    & 0.281             & 0.106              \\
4                    & 0.506             & 0.416              \\
5                    & 0.918             & 0.856              \\
\bottomrule
\end{tabular}
\end{table}

A complementary goal-directed Jacobian analysis supports, but does not formally prove, that the discrimination derives from accumulated $\partial z_{\text{pred},h} / \partial a$ composition across the five steps, not from any single horizon's residual. The 1-step head replacement of Appendix~\ref{app:audit-closure}, Table~\ref{tab:audit-ledger} directly tests this and produces catastrophic regression, consistent with the depth-is-causal interpretation.

\subsection{Where the remaining online error appears to live}
\label{app:replan-shift}

The online success rate of the IMWM evaluated recipe on OGBench-Cube / Push-T / Two-Room averages $94.9\%$ across the 12 fresh-seed cells per task (Section~\ref{sec:experiments}, Table~\ref{tab:headline}). The remaining gap appears to arise from three regimes outside the offline-energy-saturation regime audited in Appendix~\ref{app:demo-rank} on the cosine-retrieval + AR(1)-noise substrate. Because IMWM inherits the same cosine retrieval and differs only by replacing AR(1) noise with isotropic Gaussian noise, the regime list below carries interpretive weight even though we did not re-run the offline audit under the isotropic-noise geometry:
\begin{enumerate}
\item \emph{Energy-rank improvement subset.} On the $3$--$8\%$ of held-out anchors where the demonstration is not the rank-1 selection under $J$, the available offline energy headroom is small and localized; this does not bound online success. Empirically (Appendix~\ref{app:audit-closure}), eight pre-specified interventions failed to extract improvement here.
\item \emph{Demo-cloning diagnostic ceiling.} If the planner's online candidate distribution is effectively restricted to a demo-action neighborhood, the empirical success rate of exact demo playback under the same retrieval mechanism would be a diagnostic ceiling for a pure demo-cloning policy, if measured under the same retrieval protocol (we do not measure this scalar separately in this paper). It is \emph{not} a ceiling for all planners that share $J$: a non-demo action synthesized inside CEM's elite refit could succeed where the demo fails.
\item \emph{Closed-loop replan distribution shift.} At replan $k > 0$, the controller observes $(o_k, o_g)$ where $o_k$ may be off the training-anchor distribution. The retrieval bank's coverage and the cost surface's calibration may degrade. This regime is not audited by Appendix~\ref{app:demo-rank} (training-distribution anchors only).
\end{enumerate}
These are conceptual regime statements supported by Appendices~\ref{app:cem-ce}--\ref{app:rollout-depth}, not a theorem about planner-wide success rate.

\subsection{Scope of this appendix's claims}
\label{app:audit-non-claims}

This appendix does \emph{not} claim that the retained substrate cannot be improved at all. It claims only that all tested frozen-substrate interventions above the encoder failed on OGBench-Cube / Push-T / Two-Room under the IMWM evaluated recipe. It does \emph{not} claim that no learned method can beat IMWM: eight pre-specified interventions failed, but the search space of possible methods is not exhausted. It does \emph{not} claim a mechanism for the brain; the biological parallels in Appendix~\ref{app:extended-positioning} are computational-level structural analogies. It does \emph{not} apply to Reacher (which was out of scope for several interventions due to retrieval-bank availability) or to out-of-distribution generalization regimes.

\section{Component details and ablations}
\label{app:components}

This appendix gives the architectural and training specifications for the components introduced in Section~\ref{sec:method} (encoder, predictor, inverse-compatibility scorer, retrieval bank, CEM solver) and the continuous cost-weight relaxation introduced in Section~\ref{sec:method-gate}. The hard gate of Section~\ref{sec:method-gate} is the evaluated method throughout; the continuous relaxation in Appendix~\ref{app:conc-mode-defs} is a parameterization analysis that does \emph{not} replace the evaluated gate.

\subsection{LeWM encoder, projector, and predictor}
\label{app:lewm-arch}

The LeWM stack (run on the \texttt{stable-worldmodel} platform~\citep{swm}) consists of three frozen components used by IMWM: an image encoder $\phi^L : o \mapsto \zlewm{}$ producing the LeWM latent, an action-conditioned predictor $\hat M^L : (\zlewm{}, a) \mapsto \zlewm{\text{next}}$ used to roll out $\Hmacro = 5$ macro steps under candidate action chunks, and an anti-collapse projector head used during LeWM pretraining but not at evaluation time. IMWM uses these components only through the rollout cost $\mathrm{MSE}(\mathrm{rollout}_L(\zlewm{t}, a), \zlewm{g})$ in equation~\eqref{eq:hybrid-cost}; the predictor is queried but not retrained. All LeWM components are held fixed across the LeWM $h{=}5$ baseline and IMWM, so the comparison isolates IMWM's proposal- and gate-side additions.

\subsection{\texorpdfstring{Inverse-compatibility scorer $\Dinv$: InfoNCE training}{Inverse-compatibility scorer Dpsi\_inv: InfoNCE training}}
\label{app:dpsi-training}

The CMPA inverse-side stack consists of an encoder $\phi^I$ and a projector head that together produce the inverse-side latent $\zinv{}$. The inverse-compatibility scorer $\Dinv$ is a contrastive joint-compatibility score
\begin{equation*}
\Dinv(\zinv{0}, \zinv{g}, a) \;=\; \frac{q(\zinv{0}, \zinv{g}) \cdot k(a)}{\tau},
\end{equation*}
where $q$ embeds the start/goal latents into a shared key space, $k$ embeds the action chunk, and $\tau$ is a scale parameter. $\Dinv$ is trained end-to-end with $\phi^I$ on demonstration windows by InfoNCE: own-anchor demonstrations are positives, and other-anchor demonstrations in the same batch are negatives. No CEM-prior random negatives are used during training (the corresponding mixing coefficient is zero). All inverse-side components are frozen at evaluation; only the scalar score $\Dinv(\zinv{0}, \zinv{g}, a)$ is queried by IMWM's hybrid cost.

\subsection{Retrieval bank construction}
\label{app:retrieval-bank}

The retrieval bank $R = \{(k_0^{(i)}, k_g^{(i)}, a^{(i)})\}_{i=1}^{|R|}$ stores triples of cosine-normalized start/goal keys produced by the frozen intuition encoder $\phi^I$ together with the demonstration action chunks $a^{(i)} \in A^{\Hsub}$. Specifically, $\mathrm{key}^{(i)} = \mathrm{normalize}([z_0^{(i), I}; z_g^{(i), I}])$, the $\ell_2$-normalized concatenation of start and goal latents on the CMPA inverse side. In the evaluated cells $|R| = 3{,}000$ exactly (verified from \texttt{ablation\_c9\_discrete\_gate/$\ast$/eval.method.json}: \texttt{retrieval\_db\_size: 3000}, \texttt{retrieval\_score: cosine}). The bank is constructed once from the same task/seed demonstration pool used to jointly train the encoder and scorer, and is not updated at evaluation time. At evaluation, IMWM queries the bank with the current env key $\mathrm{normalize}([\zinve{t}{e}; \zinve{g}{e}])$ and returns the action chunk corresponding to the bank entry with maximum cosine similarity to the query key (top-1 nearest neighbor in the intuition encoder's latent space).

\subsection{CEM solver hyperparameters}
\label{app:cem-hyperparams}

The solver samples $S = 300$ candidate action chunks per iteration and runs $T = 30$ elite-refit iterations with top-$K = 30$ elites. The initial proposal at iteration 0 is an isotropic Gaussian $\mathcal{N}(\abase^{(e,t)}, \sigma^2 I)$ in action-chunk space with $\sigma = 1.0$ when retrieval-init is enabled; in the forward-only fallback recipe, the iteration-0 proposal is $\mathcal{N}(\mathbf{0}, \sigma^2 I)$ with the same $\sigma$. Candidate samples are not explicitly truncated to the action box at sampling time; the environment handles any out-of-bounds projection at execution. At every CEM iteration the solver forces candidate slot 0 to the current iteration mean (which is $\abase^{(e,t)}$ at iter 0 in retrieval-enabled regimes). Elite refit updates both the mean and covariance of the proposal to the elite empirical statistics; the solver returns the iter-30 elite mean as the action chunk for the replan step. Implementation note: the deployed IMWM evaluations set \texttt{ar1\_rho = 0.0} in the underlying \texttt{AR1CEMSolver}, which makes the noise sampler $\mathrm{N}(0, I)$ IID across the chunk dimensions (no temporal correlation); we therefore describe the proposal as isotropic Gaussian throughout the paper.

\subsection{Full IMWM cell procedure, diagnostic algebra, and compute footprint}
\label{app:iwm-procedure}

This subsection collects the implementation details behind \S\ref{sec:method}:
the $z$-scoring used in the hybrid cost~\eqref{eq:hybrid-cost}, the exact
cell-level diagnostic features of the reliability gate, the cell-level
procedure, and the compute footprint.

\paragraph{Per-anchor $z$-scoring.}
For values $\{x_i : i \in S\}$ over the CEM candidate set,
\begin{equation*}
\zS(x_i) = \frac{x_i - \mathrm{mean}_{j \in S}\, x_j}{\mathrm{std}_{j \in S}\, x_j + \varepsilon},
\qquad \varepsilon = 10^{-8},
\end{equation*}
with population standard deviation, so both axes of~\eqref{eq:hybrid-cost} are
unit-scale across the candidate set at every iteration before composition. The
negation $-\Dinv$ converts the intuition score to a cost.

\paragraph{Cell-level diagnostic features.}
At the first replan in cell $c$, for each environment $e$ IMWM evaluates the
intuition scorer at the retrieved chunk $\abase^{(e,0)}$ and at a fixed-size set
$\mathcal{N}_{\mathrm{neutral}}^e$ of standard-Gaussian random chunks (separate
from the CEM candidate set):
\begin{align*}
\mathrm{retrieved\_dpsi}_e & = \Dinv(\zinve{0}{e}, \zinve{g}{e}, \abase^{(e,0)}), \\
\mathrm{neutral\_q95}_e     & = q_{95}\big\{ \Dinv(\zinve{0}{e}, \zinve{g}{e}, a) : a \in \mathcal{N}_{\mathrm{neutral}}^e \big\}, \\
\mathrm{neutral\_MAD}_e     & = \mathrm{MAD}\big\{ \Dinv(\zinve{0}{e}, \zinve{g}{e}, a) : a \in \mathcal{N}_{\mathrm{neutral}}^e \big\}, \\
\rinve                   & = \frac{\mathrm{retrieved\_dpsi}_e - \mathrm{neutral\_q95}_e}{\mathrm{neutral\_MAD}_e + \varepsilon},
\qquad
\rlage                   = \text{lag-1 autocorr.\ of } \abase^{(e,0)},
\end{align*}
and the cell features are $\rinv(c) = \mathrm{median}_{e}\, \rinve$,
$\rlag(c) = \mathrm{median}_{e}\, \rlage$. Thus $\rinve$ measures how much
better the retrieved chunk scores than a high percentile of random chunks,
normalized by their dispersion; $\rlage$ is a smoothness diagnostic of the
retrieved chunk and does \emph{not} parameterize the (isotropic) proposal.

\paragraph{Cell-level procedure.}
\begin{figure}[h]
\centering
\fbox{\parbox{0.96\textwidth}{
\small
\textbf{Procedure IMWM-Cell.} \\
\textbf{Inputs:} frozen intuition encoder $\phi^I$ and scorer $\Dinv$, frozen
world model, retrieval bank $R$, neutral-chunk sampler
$\mathcal{N}_{\mathrm{neutral}}$, thresholds $(\Tinv,\Tlag)$. Per cell
$c = (\mathrm{task}, ds, ss)$ with $E$ environments.

\medskip
\textit{Cell setup} (once per $c$): the diagnostic pre-pass runs for
\emph{every} environment, including those whose cell will route to the
fallback; only \emph{planning-time} retrieval is later disabled in fallback.
\begin{enumerate}[topsep=2pt,itemsep=2pt,leftmargin=18pt]
\item For each $e$: encode $(\zinve{0}{e}, \zinve{g}{e})$ and $(\zlewme{0}{e}, \zlewme{g}{e})$; retrieve $\abase^{(e,0)}$ by a single $|R|$-way cosine lookup~\eqref{eq:retrieval-init}; compute $\rinve, \rlage$ using $\Dinv$ and $\mathcal{N}_{\mathrm{neutral}}^e$.
\item Aggregate $\rinv(c), \rlag(c)$ as medians; apply the $(\Tinv,\Tlag)$ rules of \S\ref{sec:method-gate} to fix $\mathrm{recipe}(c)$.
\end{enumerate}
\textit{Per-episode closed-loop planning} (each $e$): at each replan $t$,
re-encode; if the recipe enables retrieval, recompute $\abase^{(e,t)}$ and
initialize CEM at $\mathcal{N}(\abase^{(e,t)}, \sigma^2 I)$, else at
$\mathcal{N}(\mathbf{0}, \sigma^2 I)$, $\sigma=1$; run $T=30$ CEM iterations of
$S=300$ candidates scored by~\eqref{eq:hybrid-cost} with the recipe's weights
(slot 0 forced to the current mean; top-$K=30$ elite refit); execute the
iter-$T$ elite mean.
}}
\label{alg:iwm-cell}
\end{figure}

\paragraph{Compute footprint.}
An IMWM cell matches the world-model-only CEM sampling budget ($S = 300$,
$T = 30$, top-$K = 30$, $\Hsub = 5$). IMWM adds, over that baseline:
\textbf{(i)} one $|R|$-way cosine dot product per retrieval-enabled replan
(plus one per environment in the cell-setup pre-pass and one scorer evaluation
on the retrieved chunk); and \textbf{(ii)} when $\alphainv > 0$, one $\Dinv$
forward pass per candidate per iteration ($S \times T = 9{,}000$ scorer calls
per replan). Both are batched and small relative to the world-model rollout
($S \times T \times \Hmacro$ predictor calls per replan). The reported headline
comparison equalizes CEM sampling budget, not wall-clock; see
Appendix~\ref{app:runtime-audit}.

\subsection{Threshold-audit protocol and dev/headline overlap}
\label{app:gate-thresholds}

We froze the discrete-gate thresholds $(\Tinv, \Tlag) = (0.05, 0.3)$
before reporting the 48-cell headline grid of Table~\ref{tab:headline}.
Two distinct cell pools are relevant here; we list them explicitly to
prevent ambiguity: the routing diagnostic grid and the 48-cell headline grid.

\paragraph{Routing diagnostic grid (24 cells; Table~\ref{tab:routing}, Figure~\ref{fig:routing-plane}).}
$\{$cube, pusht, tworoom, reacher$\} \times ds \in \{3, 5, 7, 9, 11, 13\}$
at $ss = 42$ ($6 \times 4 = 24$ cells). The reliability gate's per-task
$\rinv(c), \rlag(c)$ ranges and the 24/24 routing match reported in
Table~\ref{tab:routing} are computed on this grid. All 24 routing
diagnostic cells appear in the headline grid as well; we do not claim
audit/headline disjointness on this pool.

\paragraph{Threshold-audit decision.}
$(\Tinv, \Tlag) = (0.05, 0.3)$ were chosen so that the discrete partition
of the routing diagnostic grid matches the pre-specified
recipe-cluster assignment per task (trusted hybrid for cube/pusht;
intuition-dominant for tworoom; forward-only fallback for reacher).
The thresholds were frozen \emph{before} the headline 48-cell evaluation
was run; we did not retune them after observing the headline aggregates.

The transparent statement of overlap on the routing diagnostic grid is
intentional: the gate's thresholds were calibrated on cells that also
appear in the headline aggregates.

\subsection{Per-task gate-threshold margins (one-axis perturbation)}
\label{app:gate-margins}

Table~\ref{tab:gate-margins} reports the minimum per-task margin
between the cell-level diagnostic and its frozen threshold, computed
over the 48 headline IMWM cells (the same population as
Table~\ref{tab:headline}, not the 24-cell routing diagnostic of
Table~\ref{tab:routing}). For recipe C (forward-only fallback), the
routing decision is $\rinv \leq \Tinv$ alone; $\Tlag$ does not enter
the gate's decision for cells routed to C, so we mark the $\Tlag$
column as N/A for Reacher with the observed margin in parentheses.

The aggregate one-axis stability bound is therefore $|\delta \Tinv| <
0.0376$ (over all 48 cells) and, restricted to the 36 recipe-A/B cells,
$|\delta \Tlag| < 0.3136$. These are one-axis bounds; joint
$(\delta \Tinv, \delta \Tlag)$ perturbations have smaller stability
margins.

\begin{table}[t]
\centering
\small
\caption{Per-task gate-threshold margins on the 48 headline cells.
$\min |\rinv - \Tinv|$ and $\min |\rlag - \Tlag|$ are taken over the
12 cells per task. For Reacher's recipe C, $\Tlag$ is recipe-inactive;
the lag-margin column is observed-only (N/A as a one-axis stability
bound) and shown in parentheses. Source:
\texttt{ablation\_c9\_discrete\_gate/*/eval.gate.json}.}
\label{tab:gate-margins}
\begin{tabular}{llcrr}
\toprule
\textbf{Task} & \textbf{Recipe} & $n$ & $\min |\rinv - \Tinv|$ & $\min |\rlag - \Tlag|$ \\
\midrule
OGBench-Cube      & A (trusted hybrid)        & 12 & $0.424$ & $0.316$ \\
Push-T    & A (trusted hybrid)        & 12 & $0.833$ & $0.489$ \\
Two-Room  & B (intuition-dominant)    & 12 & $0.043$ & $0.3137$ \\
Reacher   & C (forward-only fallback) & 12 & $0.0376$ & $\text{N/A}\ (0.311)$ \\
\bottomrule
\end{tabular}
\end{table}

\subsection{Approximate operational timing}
\label{app:runtime-audit}

Table~\ref{tab:runtime-audit} reports the median and interquartile
range of stored per-cell wall-clock timings for c1 and IMWM across the
48 headline cells. \textbf{This is not a controlled runtime
comparison}: the c1 stable-WM eval harness used
\texttt{solver.batch\_size${=}1$} while the IMWM eval harness used
\texttt{solver.batch\_size${=}8$}, a deliberate harness choice for
IMWM's GPU utilization. The wall-clock seconds are therefore stored
operational timings, not a matched runtime benchmark. We provide them
here for transparency only.

\begin{table}[t]
\centering
\small
\caption{Approximate logged wall-clock timings per cell (seconds),
median with IQR in brackets, over the 48 headline cells (12 per task).
\textbf{Not} a controlled runtime comparison: c1 evals use
\texttt{solver.batch\_size${=}1$}; IMWM evals use
\texttt{solver.batch\_size${=}8$}.}
\label{tab:runtime-audit}
\begin{tabular}{lccc}
\toprule
\textbf{Task} & c1 median [IQR] (s) & IMWM median [IQR] (s) & approx.\ ratio \\
\midrule
OGBench-Cube      & $331\ [323, 338]$ & $376\ [357, 396]$ & $1.14\times$ \\
Push-T    & $278\ [264, 281]$ & $306\ [288, 343]$ & $1.10\times$ \\
Two-Room  & $340\ [309, 349]$ & $333\ [299, 346]$ & $0.98\times$ \\
Reacher   & $362\ [331, 372]$ & $357\ [333, 373]$ & $0.99\times$ \\
\bottomrule
\end{tabular}
\end{table}

\subsection{Add1 reduced sentinel sweep: harness sanity table}
\label{app:add1}

Section~\ref{sec:exp-add-budget} reports the Add1 reduced sentinel
$(S, T) = (300, \{15, 30, 60\})$ slice. The c1 baselines used in Add1-A
were re-run through the IMWM eval harness rather than re-used from the
published \texttt{.stable-wm} artifacts, because the IMWM harness uses
\texttt{solver.batch\_size = 8} while the \texttt{.stable-wm} c1 pipeline
uses \texttt{solver.batch\_size = 1}. Before trusting the new c1 path,
we ran a T${=}30$ sanity reproduction on the three cells used in Add1
and required that the new c1 path match the published cells within the
$\pm 2$~pp single-episode granularity of $N = 50$.

\begin{table}[t]
\centering
\small
\caption{Add1 c1 harness sanity: published \texttt{.stable-wm} c1 success
rate (under \texttt{solver.batch\_size = 1}) vs.\ new c1 path through the
IMWM eval harness (under \texttt{solver.batch\_size = 8}) at $T = 30$.
All deltas are at the single-episode granularity ($\pm 2$~pp $= \pm 1$
episode out of $50$); the new c1 path is acceptable as the Add1 c1
reference. Source: \texttt{ablation\_c1\_lewm\_iwmharness/}.}
\label{tab:add1-sanity}
\begin{tabular}{lrrr}
\toprule
\textbf{Cell} & \texttt{.stable-wm} c1 (\%) & new c1 path (\%) & $\Delta$ (pp) \\
\midrule
cube\_ds3\_ss42      & $68.0$ & $66.0$ & $-2.0$ \\
cube\_ds5\_ss1       & $64.0$ & $64.0$ & $\phantom{+}0.0$ \\
tworoom\_ds3\_ss42   & $88.0$ & $90.0$ & $+2.0$ \\
\bottomrule
\end{tabular}
\end{table}

\subsection{\texorpdfstring{Continuous cost-weight relaxation: full definitions ($\concc$, $\alphainv(c)$, $\beff(c)$)}{Continuous cost-weight relaxation: full definitions (conc(c), alpha\_inv(c), beta\_eff(c))}}
\label{app:conc-mode-defs}

Using the per-env quantities of Section~\ref{sec:method-gate}, we define
\begin{equation}
\label{eq:conc-c}
\concc = \frac{\mathrm{mean}_{e \in c} \, \mathrm{retrieved\_dpsi}_e - \mathrm{mean}_{e \in c} \, \mathrm{neutral\_q95}_e}{\mathrm{std}_{e \in c} \, \mathrm{retrieved\_dpsi}_e + \varepsilon},
\end{equation}
where $\mathrm{retrieved\_dpsi}_e$ is the top-1 retrieved chunk's $\Dinv$ score at call 0 (as in Section~\ref{sec:method-gate}), $\mathrm{neutral\_q95}_e$ is the per-env $q_{95}$ of $\Dinv$ over standard-Gaussian neutral candidates, and the standard deviation in the denominator is over envs $e \in c$. The continuous relaxation maps $\concc$ into cost weights via two sigmoids and one compositional step:
\begin{align}
\alphainv(c) & = \sigma\!\big((\concc - \Ttrust)/\kappa\big), \label{eq:alpha-inv} \\
w_{\mathrm{focus}}(c) & = \sigma\!\big((\concc - \Tfocus)/\kappa\big), \label{eq:w-focus} \\
\beff(c)     & = (1 - \alphainv(c)) + \alphainv(c) \cdot \big(\blow + (\bhigh - \blow) \cdot w_{\mathrm{focus}}(c)\big). \label{eq:beta-eff}
\end{align}

The three global scalars $(\kappa, \blow, \bhigh) = (0.005, 0.1, 3.0)$ are fixed across all tasks. The two thresholds $(\Ttrust, \Tfocus) = (0.02285, 0.24425)$ are calibration-set functionals of a 24-cell audit: cluster midpoints between adjacent recipe-cluster medians of $\concc$. The cost-weight values used by the discrete recipes of Section~\ref{sec:method-gate} are recovered as the $\kappa \to 0$ saturating limits of equations~\eqref{eq:alpha-inv}--\eqref{eq:beta-eff}: $(\alphainv, \beff) \in \{(1, 3), (1, 0.1), (0, 1)\}$. The evaluated routing boundary itself is the two-feature gate of Section~\ref{sec:method-gate}; the proposal-side retrieval switch (retrieval-init on/off) also remains discrete in both the evaluated method and the relaxation.

\section{Search-space audit: per-intervention details}
\label{app:audit-ledger}

This appendix expands Table~\ref{tab:audit-ledger} from Appendix~\ref{app:audit-closure} with per-intervention descriptions and pre-specified closure criteria. Each entry summarizes the intervention's design, the offline gate or online smoke that closed it, and the empirical outcome on OGBench-Cube / Push-T / Two-Room.

\paragraph{Heteroscedastic forward-density head.}
\emph{Layer:} cost composition / precision head. The forward-rollout MSE term is replaced with a calibrated heteroscedastic Gaussian negative log-likelihood: $\tfrac{1}{2}(\epsilon^2 / \sigma^2) + \tfrac{1}{2} \log \sigma^2$ where a small density head trained on frozen-LeWM residuals predicts $\sigma^2$ at each step. \emph{Pre-specified gate:} validation NLL improvement on demonstration windows. \emph{Outcome:} The head improves validation NLL by 7--34 units per sample with calibrated Mahalanobis statistics, but substituting it for the $z$-scored MSE term in equation~\eqref{eq:hybrid-cost} regresses planner ranking (combined 6 wins / 7 ties / 14 losses across 27 cells). Closed on the online substitution gate; summarized in Section~\ref{sec:exp-ablations}.

\paragraph{Inverse-scorer retrain with retrieval hard negatives.}
\emph{Layer:} inverse-scorer retraining. The InfoNCE training of $\Dinv$ is augmented with retrieval-mined hard negatives: action chunks from nearby, but non-own-anchor, demonstrations in the cosine-key space used by IMWM's bank-retrieval criterion. \emph{Pre-specified gate:} per-task 2~pp tie threshold against the IMWM evaluated recipe on a fresh-grid extension. \emph{Outcome:} Neutral or marginal regression; closure criterion not crossed.

\paragraph{Alternative forward-CPC residual cost.}
\emph{Layer:} cost composition + precision (auxiliary-predictor residual). The hybrid cost is augmented with an auxiliary-predictor's forward-CPC residual as a precision-weighted term. \emph{Pre-specified gate:} online substitution on the 12-cell headline grid. \emph{Outcome:} $-12$~pp on OGBench-Cube, $-31$~pp on Push-T. Closed.

\paragraph{LeWM-residual precision-weighted cost.}
\emph{Layer:} cost composition + precision (LeWM residual). Similar to the alternative forward-CPC residual cost but with the LeWM-predictor residual itself supplying the precision weight. \emph{Pre-specified gate:} same online substitution gate. \emph{Outcome:} $-2.7$~pp on OGBench-Cube, $-4.0$~pp on Push-T, $-20.7$~pp on Two-Room. Closed.

\paragraph{1-step head replacement of 5-step rollout.}
\emph{Layer:} forward-primitive replacement. The 5-step LeWM rollout in $\mathrm{MSE}$ is replaced with a 1-step inverse-action head; the rest of the hybrid cost is unchanged. \emph{Pre-specified gate:} per-task 2~pp tie threshold. \emph{Outcome:} $-25.3$~pp on OGBench-Cube, $-68$~pp on Push-T. Catastrophic regression. Closed. This intervention is the direct empirical test of the depth-is-causal hypothesis discussed in Appendix~\ref{app:rollout-depth}.

\paragraph{Stochastic-slot proposal augmentation ($q_\phi$ residual).}
\emph{Layer:} stochastic-slot proposal augmentation. The CEM proposal is augmented with samples from a learned $q_\phi$ residual centered on the retrieval mean; the deterministic mean slot still carries the retrieved chunk. \emph{Pre-specified gate:} online substitution + sign-test on $4 \times 50$-episode dev cells. \emph{Outcome:} $+0.22$~pp aggregate, sign-test $p = 1.000$ over 450 paired episodes. Statistically neutral; the intervention did not cross the closure criterion. (CEM elite refit averages out small iter-0 proposal shifts of this magnitude.)

\paragraph{Deterministic-slot retrieval reranking.}
\emph{Layer:} deterministic-slot retrieval reranking. The retrieved chunk in candidate slot 0 is replaced by a direct-$J$ argmin selection over the top-$K$ cosine-retrieved chunks (a ``best-of-$S$'' reranking gate). \emph{Pre-specified gate:} a center-effect closure test on OGBench-Cube / Push-T / Two-Room dev anchors comparing post-refit best-cost stability of cosine top-1 versus direct-$J$ reranking. \emph{Outcome:} Cosine top-1 has refit-stability ($0.0$--$0.26 z$ on post-refit best-cost) that direct-$J$ reranking lacks; the gate identifies cosine top-1 as the correct anchor under CEM refit. Closed before online evaluation.

\paragraph{Rollout-residual readout adapter.}
\emph{Layer:} adapter over rollout-residual readout. A small adapter is trained on the frozen encoder + frozen predictor stack to read out an alternative latent cost from the rollout residual. \emph{Pre-specified gate:} demo-rank audit (Appendix~\ref{app:demo-rank}) verifying that the alternative readout would have rank headroom over the retained hybrid cost. \emph{Outcome:} The audit shows $\le 2$ rank-headroom on 92.0--97.0\% of held-out anchors (Appendix~\ref{app:demo-rank}, Table~\ref{tab:demo-rank}); the alternative readout cannot improve over a near-saturated baseline on the planner's cost axis. Closed before online evaluation.

\paragraph{Audit closure summary.}
The eight interventions span six conceptually distinct layers: (i)~cost composition / precision head, (ii)~inverse-scorer retraining, (iii)~forward-primitive replacement, (iv)~stochastic-slot proposal augmentation, (v)~deterministic-slot retrieval reranking, (vi)~adapter over rollout-residual readout. Of these, six were closed by online substitution under the evaluated recipe and two were closed before online evaluation by pre-specified offline gates. The audit is a closure result on the substrate of IMWM's evaluated recipe and the three retrieval-task benchmarks; it does not claim that no method outside this audit could improve over IMWM, and it does not apply to Reacher (where retrieval is uninformative and the fallback recipe routes the cell to forward-only LeWM/CEM).

\section{Reproducibility: anonymized artifact contents and release plan}
\label{app:reproducibility}

This appendix describes the contents of the artifact bundle and the release plan. All filesystem paths and internal identifiers have been anonymized; the bundle ships with relative paths only.

\paragraph{Code.}
The artifact bundle contains: (i)~the IMWM planner implementation (proposal, hybrid cost, reliability gate); (ii)~the LeWM stack used as the world-model substrate (publicly released by the LeWM authors; we use the released checkpoints verbatim); (iii)~the intuition-model (inverse-side) training code for $\Dinv$; (iv)~evaluation scripts that reproduce the 12-cell paired evaluation and the 24-cell diagnostic grid; (v)~plotting scripts for Figures~\ref{fig:paired-scatter}--\ref{fig:routing-plane} and the appendix tables. Full component specifications are in Appendices~\ref{app:components}, \ref{app:cem-hyperparams}, and~\ref{app:iwm-procedure}.

\paragraph{Pretrained checkpoints.}
The bundle includes the frozen LeWM encoder, predictor, and projector checkpoints (referenced from the public LeWM release) and the frozen intuition-model inverse-side encoder/projector and $\Dinv$ checkpoint trained on the demonstration windows. No additional training is required at evaluation time; all components are loaded as released.

\paragraph{Retrieval bank.}
The bundle includes the constructed retrieval bank $R$ as a serialized file containing the $(z_0^{(i), I}, z_g^{(i), I}, a^{(i)})$ triples. Bank construction is one-shot and the bank is not updated at evaluation; reconstructing the bank from raw demonstrations is also documented in the release.

\paragraph{Evaluation grid and seeds.}
The 12-cell-per-task headline grid is fully specified in Section~\ref{sec:exp-setup}: $(ds, ss) \in \{3, 5, 7\} \times \{42, 1, 2\}$ plus $\{(9, 42), (11, 42), (13, 42)\}$ per task. Per-cell starts (the random episode indices) are recorded in the artifact bundle so that paired evaluation against LeWM $h{=}5$ can be reproduced exactly. The 24-cell diagnostic grid is documented similarly.

\paragraph{Evaluation logs.}
The bundle includes per-cell evaluation logs with the per-episode success/failure outcomes, the per-cell gate diagnostics $\rinv(c)$ and $\rlag(c)$, and the selected $\mathrm{recipe}(c)$. It also includes per-episode method sidecars that record the active runtime flags. The IMWM-active sidecar fields are \texttt{alpha\_inv}, \texttt{beta\_eff}, \texttt{retrieval\_init\_on}, \texttt{proposal\_type=isotropic}, \texttt{sigma=1.0}, \texttt{retrieval\_score=cosine}, \texttt{gate\_retrieval\_score\_mode=cosine}, \texttt{planning\_retrieval\_score\_mode=cosine} (recipes A, B) or \texttt{planning\_retrieval\_score\_mode=none} (recipe C), and \texttt{retrieval\_db\_size=3000}. Legacy sidecars may also carry a $\rho$ (\texttt{ar1\_rho}) field that records \texttt{0.0} for all IMWM evaluations; it is a historical pass-through from the underlying AR(1)-capable solver class and is not part of the IMWM method definition. These logs allow independent re-derivation of Tables~\ref{tab:headline} and~\ref{tab:routing}.

\paragraph{Release plan.}
We submit the anonymized code, scripts, retrieval bank, and evaluation logs as supplementary material. The frozen LeWM checkpoints are released at the LeWM project's public location (cited in the main paper); the intuition-model inverse-side checkpoint is included in the artifact bundle. After acceptance, the de-anonymized release will mirror the supplementary bundle on a public repository with the camera-ready citation. Local filesystem paths from the development environment have been replaced with relative artifact-bundle paths throughout the released code.

\paragraph{What this bundle does not include.}
The bundle does not include training trajectories or raw demonstration episodes beyond what is needed to reconstruct the retrieval bank and to re-evaluate the 12-cell grid. The bundle does not include training code for the LeWM stack, which is released by the LeWM authors, or for the additional intervention variants tested in Appendix~\ref{app:audit-ledger}. Those interventions are documented in the appendix, but their training artifacts are not part of this release.

\section{Extended positioning and related work}
\label{app:extended-positioning}

This appendix expands the positioning of \S\ref{sec:motivation}: the
cognitive-science motifs that inspired the intuition--world-model pairing
(\S\ref{sec:motivation-cogsci}), the full ledger of AI-side functional
analogues (\S\ref{sec:motivation-ai}), and the detailed differentiation of
IMWM from each adjacent line (\S\ref{sec:motivation-gap}).

\subsection{Cognitive-science and neuroscience analogues at the level of functional motifs}
\label{sec:motivation-cogsci}

Biological decision-making is not well-modeled by a single forward simulator.
We summarize four functional motifs from the cognitive-neuroscience literature
that are relevant for the intuition-model framing. We are explicit that these
are functional motifs in biological decision systems; we are not claiming the
brain implements any of IMWM's specific modules.

\paragraph{Arbitration between control systems.}
Decades of work argue that the brain coordinates multiple decision-making
systems whose relative influence is regulated rather than fixed.
\citet{daw05} formalize arbitration between a prefrontal-associated
model-based system and a dorsolateral-striatum-associated model-free system
by the relative uncertainty of each system's value estimates; the more
reliable system dominates at each decision.
\citet{lengyel07} extend this picture by proposing an additional
\emph{episodic} control system, grounded in hippocampal/MTL function, that
becomes useful when forward-simulation value estimates are noisy and habit
estimates have not converged. The shared functional motif is that biological
decision systems do not run a single planner at full strength; they
\emph{gate} how much each system contributes.

\paragraph{Memory access and replay for decision support.}
\citet{mattar18} propose a normative theory of which past memories are
accessed at each moment to optimize future decisions, and use it to predict
the structure of hippocampal replay events as utility-prioritized retrieval.
\citet{pfeiffer13} report that brief ($\sim$100\,ms) goal-biased forward
sequences in CA1 place cells precede goal-directed navigation in rats,
beginning at the current location and tending to terminate near the
remembered goal even for novel start-goal combinations.
\citet{stachenfeld17} model hippocampal place cells as encoding
policy-dependent \emph{predictive} maps that link current locations to
expected future locations under the current policy.
\citet{tolman48} is the historical anchor for the cognitive-map idea that
informs this lineage. The shared functional motif is that biological agents
use \emph{retrieved} or \emph{predictive} structure over experience as a
decision-time aid, not only as a value-update mechanism after the fact.

\paragraph{Hierarchical control and chunked actions.}
\citet{dezfouli13} present human behavioral evidence that goal-directed and
habitual action control are hierarchically organized: the goal-directed system
can select chunked action sequences that are then executed efficiently by
the habit system. The shared motif is that decision-making operates over
temporally extended action chunks, not only per-step actions.

\paragraph{Functional, not implementational.}
Together, these motifs do not by themselves prescribe a system design; they
establish that biological decision systems exhibit arbitration, predictive
maps, memory-driven retrieval, episodic-style control, and chunked action
selection. IMWM is best read as an engineering instantiation of analogous
functions for pixel-based goal-reaching, not as a claim about how brains
implement these functions.

\subsection{AI-side functional analogues}
\label{sec:motivation-ai}

We now turn to AI work that already realizes pieces of the intuition-model
role. We group it by mechanism rather than by chronology, and we explicitly
identify which analogues are close enough to require careful differentiation
in \S\ref{sec:motivation-gap}.

\paragraph{Episodic and memory-augmented control.}
Model-Free Episodic Control~\citep{blundell16mfec} stores high-return
state--action estimates in a non-parametric memory and acts via
nearest-neighbor lookup, learning faster than parametric deep RL on hard
Atari domains. Neural Episodic Control~\citep{pritzel17nec} extends this to
a differentiable neural dictionary over slowly-changing state embeddings
with rapidly-updated value estimates, again improving early-learning sample
efficiency. Imagination-Augmented Agents~\citep{racaniere17i2a} combine
model-free and model-based signals by feeding learned-model rollouts into a
policy network as additional context.

\paragraph{Decision-time retrieval and imitation.}
Retrieval-Augmented RL~\citep{goyal22rarl} trains a learned retrieval module
to surface elements of a dataset of past experiences (own experience,
expert demonstrations, or other sources) and condition agent behavior on
them at decision time. The recent Retrieval-Augmented Decision
Transformer~\citep{schmied24radt} introduces an external memory storing
sub-trajectories of past experience and retrieving only the relevant
sub-trajectories for the current state, removing the need to keep entire
episodes in context. IMPLANT~\citep{qi22implant} retains both a behavior-cloning imitation policy and a learned reward model at decision time and uses
test-time planning to correct compounding errors of the imitation policy
under dynamics perturbations. Earlier work on imitation-shaped
planners~\citep{choudhury18idplanning} trains planning policies by imitating
a clairvoyant oracle to bias search behavior under partial observability.

\paragraph{Learned priors over actions, skills, and trajectories.}
A complementary line treats prior data as a source of action-, skill-, or
trajectory-level priors that bias downstream agents.
SPiRL~\citep{pertsch20spirl} learns a latent embedding space of temporally
extended skills together with a skill prior from offline experience and
uses the prior to regularize maximum-entropy RL; SkiMo~\citep{shi22skimo}
plans in a skill latent space using a skill dynamics model rather than a
per-step dynamics model. \citet{tirumala22bp} develop behavior priors as
probabilistic trajectory models that capture shared movement and
interaction patterns across related tasks; \citet{galashov19infoasym}
analyze the limiting case where an information-restricted default policy
forces task-general reusable behavior under KL regularization.
PARROT~\citep{singh21parrot} pre-trains a behavioral prior from successful
trials across diverse prior tasks and uses it to accelerate downstream
RL, including on pixel-based manipulation.
Diffuser~\citep{janner22diffuser} folds offline trajectory generation and
planning into a single diffusion probabilistic model that plans by
denoising trajectories.

\paragraph{Value-aware and planning-aware model learning.}
A separate line argues that the world-model objective should reflect the
downstream value/planning objective rather than only observation
prediction. The Predictron~\citep{silver17predictron} trains an abstract
Markov-reward-process model end-to-end so that multi-step rollouts of
accumulated values approximate the true value function. Value Prediction
Networks~\citep{oh17vpn} integrate model-free and model-based RL by learning
abstract dynamics that predict option-conditional future values rather than
future observations.
Value-aware loss functions~\citep{farahmand17vaml} make the same argument
theoretically with finite-sample analysis. \citet{lambert20_objective_mismatch}
name and empirically characterize objective mismatch in model-based RL,
showing one-step likelihood is not always correlated with control
performance. \citet{hamrick21_role_of_planning} ablate
planning components in a MuZero-style agent and find that planning's
benefits vary by task and algorithmic setting.

\paragraph{CEM/MPC optimizer engineering and modern latent MPC.}
The CEM-based MPC pipeline that IMWM augments has its own line of work.
PETS~\citep{pets} combines uncertainty-aware deep dynamics ensembles with
sampling-based propagation under CEM and remains a canonical CEM-MPC +
learned-dynamics baseline. PDDM~\citep{nagabandi20pddm} couples deep dynamics
models with MPC for real-world contact-rich dexterous manipulation. The
closest optimizer-side component prior to IMWM's proposal distribution is
iCEM~\citep{pinneri20icem}, which improves CEM through temporally
correlated actions and shifted/best-elite memory across replans. CEM with
interleaved gradient steps~\citep{bharadhwaj20cemgd} and Differentiable
MPC~\citep{amos18diffmpc} represent alternative optimizer choices. The
strongest \emph{modern} latent-MPC baseline is TD-MPC~\citep{tdmpc}, which
adds a task-oriented latent dynamics model and a learned terminal value
function trained jointly by temporal difference.

\paragraph{Contrastive and inverse-dynamics representation learning.}
The training objective of IMWM's intuition model is in the lineage of
contrastive representation learning. Contrastive Predictive
Coding~\citep{vandenoord18cpc} introduces an InfoNCE-style negative-sampling
objective for predicting future samples in latent space. The Intrinsic
Curiosity Module~\citep{pathak17icm} uses an inverse-dynamics objective so
that learned features capture aspects of the environment the agent can act
on while ignoring uncontrollable factors, a functional motif directly relevant to action-relevant compatibility scoring. R3M~\citep{nair22r3m}
applies time-contrastive learning over large-scale human video to obtain a
frozen visual representation that accelerates downstream manipulation.
Contrastive RL~\citep{eysenbach22contrastiverl} shows that contrastive
representation learning applied to action-labelled trajectories yields
inner-product representations corresponding to goal-conditioned value
functions; this is the most directly relevant precedent for a
contrastive-scored start-goal-action signal.

\subsection{The novelty gap IMWM fills}
\label{sec:motivation-gap}

Read across the above body of work, the intuition-model role is not new in
isolation; what is new is the specific combination IMWM instantiates inside a
finite-budget CEM planner over a frozen latent world model. We summarize the
required differentiation by group.

\paragraph{Episodic / memory-augmented control vs.\ planning-time action-chunk
retrieval.}
MFEC and NEC use episodic retrieval to estimate \emph{value} for action
selection; the retrieved object is a value estimate, the use is
$\arg\max$-over-actions. IMWM retrieves \emph{action chunks} from
demonstrations indexed by contrastive joint-compatibility latents, uses them
to \emph{initialize} the CEM proposal distribution, and composes a hybrid
cost over a frozen latent forward predictor. There is no value memory and no
$\arg\max$-over-stored-Q-values step.

\paragraph{Retrieval / imitation at decision time vs.\ proposal-init + hybrid cost + gate.}
RA-RL conditions a value/policy network on a learned retrieval over an
experience dataset; RA-DT retrieves sub-trajectories into a Transformer
context window; IMPLANT plans at test time with a BC policy and a learned
reward model. IMWM instead retrieves action chunks as the \emph{initialization
mean} of a CEM proposal under a frozen latent world model, scores
candidates by a hybrid cost composed from the forward latent-MSE and a
contrastive intuition score under per-cell z-scoring, and gates reliance on
this signal by cell-level reliability diagnostics. IMWM has no task reward,
no value function, and no in-context Transformer policy at planning time.

\paragraph{Behavior / skill / trajectory priors vs.\ planning-time compatibility
score.}
SPiRL, Behavior Priors, PARROT, and SkiMo learn priors over actions, skills,
or trajectories from prior data and use them to bias downstream RL or
planning. Diffuser learns a diffusion trajectory model that doubles as a
planner. IMWM instead trains a goal-conditioned compatibility score from
demonstrations and uses it at decision time to initialize and score CEM
candidates under a frozen latent world model, with per-cell reliability
gating. The intuition model is not used to fine-tune a policy or to
generate trajectories; it is consumed inside CEM as a cost term.

\paragraph{CEM/MPC optimizer engineering vs.\ retrieval-derived proposal mean.}
iCEM improves CEM via temporally correlated (colored-noise) actions and
shifted/best-elite memory across replans. IMWM takes a different route: it
\emph{replaces the proposal mean} with the bank chunk selected by
cosine nearest-neighbor lookup in the intuition encoder's latent space
and samples isotropically (with no temporal correlation) around that
mean. The optimizer-side innovation is on a different axis from iCEM's,
and IMWM is not built on top of iCEM's colored-noise mechanism.
TD-MPC is the closest \emph{modern} latent-MPC baseline class: it learns a
task-oriented latent dynamics model and a terminal value function jointly
by TD. IMWM keeps a frozen forward-prediction-trained world model with no reward and
no learned terminal value at planning time. It composes the hybrid cost above and adds the reliability gate.

\paragraph{Contrastive goal-conditioned scoring vs.\ compatibility cost inside
finite-budget CEM.}
Contrastive RL connects action-labeled trajectory contrastive learning to
goal-conditioned value functions; the contrastive scorer there is
interpreted as a value function. IMWM instead uses an InfoNCE-style
start-goal-action scorer as a planning-time compatibility cost term inside
finite-budget CEM, composed with the world model's terminal latent-MSE and gated by
reliability diagnostics; it is not trained or interpreted as a terminal
value function.

\paragraph{Summary of the gap.}
Our literature review did not identify prior work that combines this specific
set of design choices: \emph{(a)} frozen latent world-model planning with terminal
latent-MSE as one cost component, \emph{(b)} a separately-trained
contrastive start-goal-action compatibility scorer as a second cost
component, \emph{(c)} cosine nearest-neighbor retrieval of a demonstration action
chunk in the intuition encoder's latent space, used as the
initialization mean of an isotropic-Gaussian CEM proposal, and
\emph{(d)} a per-cell reliability gate over cell-level diagnostics that
selects among discrete recipes for the cost weights and the proposal
distribution. The diagnosis in \S\ref{sec:diagnosis} identifies exactly the
volume-bottleneck axis that this combination is built to attack;
\S\ref{sec:method} defines IMWM precisely.

\section{Additional experimental results}
\label{app:extra-results}

This appendix collects the per-cell results, full diagnostic tables, and
control experiments summarized in \S\ref{sec:experiments}.

\subsection{Per-cell paired results}
\label{app:extra-paired}

Figure~\ref{fig:paired-scatter} shows the per-cell paired success scatter
underlying the Table~\ref{tab:headline} totals. OGBench-Cube is the largest and most
consistent gain (IMWM improves every cell, with lower spread, SD $1.9$ vs.\
$5.7$); Two-Room improves every cell ($+11.5$~pp mean). Push-T is a smaller,
high-baseline win (8 wins, 1 tie, 3 losses; mean $+2.8$~pp, just above the $\pm 2$~pp tie band). On Reacher the world-model-only baseline is already
competitive and IMWM routes to the forward-only fallback recipe, leaving a $+0.7$~pp near-tie (run-to-run CEM variation).

\begin{figure}[h]
\centering
\includegraphics[width=0.85\textwidth]{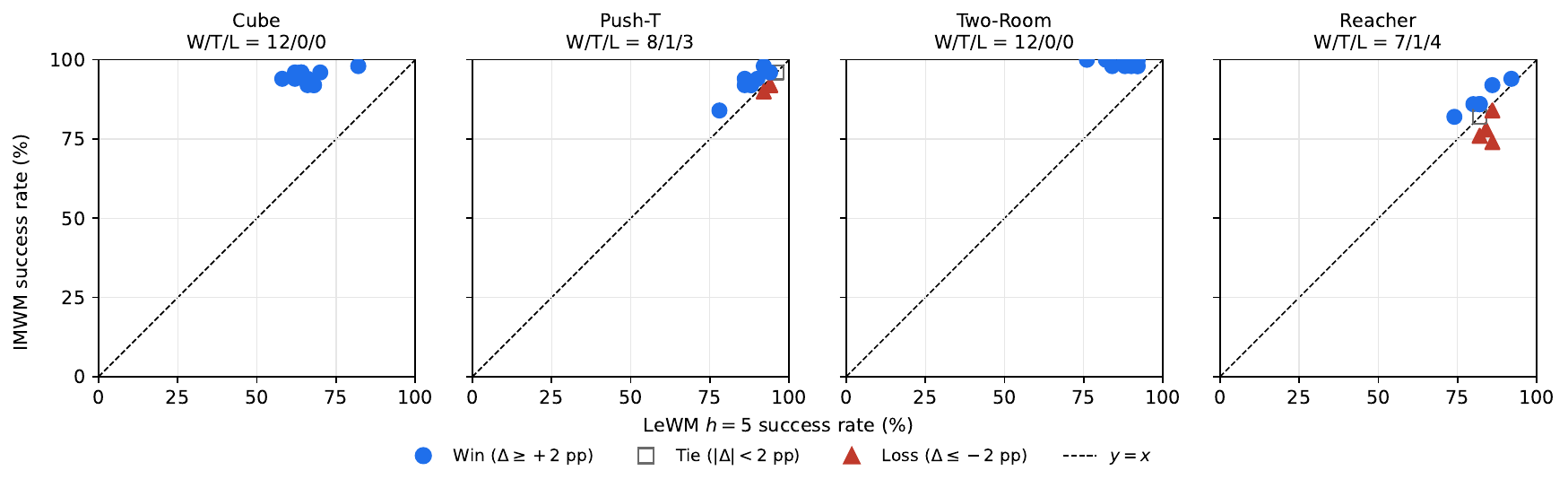}
\caption{Per-cell paired success scatter (IMWM vs.\ the world-model-only
baseline), one panel per task. Each point is one of $12$ fresh-seed cells;
color codes W/T/L under the strict $|\Delta| < 2$~pp tie threshold. Generated
from the canonical $300/30/30$ evaluation data underlying
Table~\ref{tab:headline}.}
\label{fig:paired-scatter}
\end{figure}

\subsection{Gate diagnostic ranges}
\label{app:extra-routing}

Table~\ref{tab:routing} reports the per-task ranges of the two gate diagnostics
across the 24-cell diagnostic grid, with the routed recipe under the frozen
thresholds. No diagnostic cell straddles a threshold boundary, which is the
source of the $24/24$ routing claim in \S\ref{sec:exp-routing}.

\begin{table}[h]
\centering
\small
\caption{Per-task ranges of the gate diagnostics across the 24-cell diagnostic
grid (4 tasks $\times$ $ds \in \{3,5,7,9,11,13\}$ at $ss{=}42$), with the routed
recipe under frozen thresholds $(\Tinv, \Tlag) = (0.05, 0.3)$. Source: per-cell
gate sidecars from the deployed planner.}
\label{tab:routing}
\begin{tabular}{lccc}
\toprule
\textbf{Task} & $\rinv$ range & $\rlag$ range & Routed recipe \\
\midrule
OGBench-Cube     & $[+0.47, +0.57]$ & $[+0.62, +0.64]$ & trusted hybrid ($\beff{=}3$) \\
Push-T   & $[+0.88, +1.13]$ & $[+0.79, +0.83]$ & trusted hybrid ($\beff{=}3$) \\
Two-Room & $[+0.09, +0.18]$ & $[-0.08, -0.01]$ & intuition-dominant ($\beff{=}0.1$) \\
Reacher  & $[-0.05, +0.01]$ & $[-0.07, -0.01]$ & forward-only fallback \\
\bottomrule
\end{tabular}
\end{table}

\subsection{Roles of the two models: full table}
\label{app:extra-diag2}

This expands \S\ref{sec:exp-diag-intuition-elite}. The three arms are
\emph{intuition-only} ($\alphainv{=}1, \beff{=}0$; retrieval-init on, isotropic
CEM, no world-model rollout), \emph{world-model-only} (retrieval-init on,
isotropic CEM, no intuition term, i.e.\ legacy condition \texttt{c2} on
Two-Room/Push-T/OGBench-Cube and the \texttt{c1}/fallback arm on Reacher because
the gate routes Reacher to recipe C), and \emph{full IMWM} (legacy \texttt{c9}, the
deployed gate-driven planner). The intuition-only arm is a new 48-cell run; the
other arms reuse existing results. Paired-cell bootstrap CIs are over the 12
cells per task ($B{=}10{,}000$).

\begin{table}[h]
\centering
\small
\caption{Intuition-only vs.\ world-model-only vs.\ full IMWM on the headline
48-cell grid (success rate, \%; $n{=}12$/task). Gate-faithful semantics: the
world-model-only column is \texttt{c2} on Two-Room/Push-T/OGBench-Cube and the
\texttt{c1}/fallback on Reacher. Intuition-only on Reacher is a counterfactual
stress test, not a deployed configuration.}
\label{tab:diag2-full}
\resizebox{\textwidth}{!}{%
\begin{tabular}{lcccrr}
\toprule
\textbf{Task} & intuition-only & WM-only (src) & full IMWM & $\bar{\Delta}$ IMWM$-$intuition & $\bar{\Delta}$ IMWM$-$WM \\
\midrule
Two-Room     & $98.3 \pm 1.4$ & $98.2 \pm 2.6$ (c2)    & $99.2 \pm 1.0$ & $+0.8\ [+0.2, +1.5]$    & $+1.0\ [-0.3, +2.7]$ \\
Reacher      & $53.2 \pm 7.5$ & $83.2 \pm 4.1$ (c1/fb) & $83.8 \pm 5.7$ & $+30.7\ [+25.3, +36.7]$ & $+0.7\ [-2.8, +3.7]$ \\
Push-T       & $60.0 \pm 6.5$ & $91.5 \pm 3.8$ (c2)    & $92.7 \pm 3.4$ & $+32.7\ [+28.3, +36.2]$ & $+1.2\ [-1.0, +3.5]$ \\
OGBench-Cube & $83.8 \pm 5.4$ & $88.3 \pm 3.0$ (c2)    & $94.7 \pm 1.9$ & $+10.8\ [+8.0, +14.2]$  & $+6.3\ [+4.2, +8.3]$ \\
\bottomrule
\end{tabular}}
\end{table}

Full IMWM materially exceeds intuition-only on OGBench-Cube, Push-T, and Reacher, while
Two-Room is practically saturated by the intuition-only score; this rules out
the intuition score as a uniform standalone replacement for the world model.
OGBench-Cube is the clearest case that intuition adds value beyond the world model
($+6.3$~pp, $95\%$ CI excludes zero); on Push-T and Two-Room the
world-model-only arm is already near-saturated so the hybrid is statistically
tied, and on Reacher the gate routes to the forward-only fallback recipe,
leaving a near-tie (the $+0.7$~pp is run-to-run CEM variation).

\subsection{CEM budget scaling (reduced sentinel sweep)}
\label{sec:exp-add-budget}

We test whether IMWM's advantage is a simple CEM-iteration-budget artifact by
fixing $S = 300$ and varying $T \in \{15, 30, 60\}$. \textbf{Add1-A}
(Table~\ref{tab:add1-a}) compares the world-model-only baseline against full
IMWM on two sentinel cells; \textbf{Add1-B} (Table~\ref{tab:add1-b}) tracks the
OGBench-Cube complementarity gap (full IMWM vs.\ the world-model-only retrieval arm) on
the cell with the largest such gap. The baseline was re-run at $T{=}30$ through
the IMWM eval harness as a sanity check and matched the published cells within
the $\pm 2$~pp single-episode granularity (Appendix~\ref{app:add1}).

\begin{table}[h]
\centering
\small
\caption{Add1-A: world-model-only baseline vs.\ IMWM at $S{=}300$,
$T \in \{15,30,60\}$ on two sentinel cells (success rate, \%).}
\label{tab:add1-a}
\begin{tabular}{llrrrr}
\toprule
\textbf{Cell} & \textbf{Arm} & $T{=}15$ & $T{=}30$ & $T{=}60$ & $\Delta$ (IMWM$-$base) \\
\midrule
cube\_ds3\_ss42    & baseline & $70.0$  & $66.0$  & $68.0$  & N/A \\
                   & IMWM      & $94.0$  & $92.0$  & $96.0$  & $+24$ / $+26$ / $+28$ \\
\midrule
tworoom\_ds3\_ss42 & baseline & $94.0$  & $90.0$  & $90.0$  & N/A \\
                   & IMWM      & $100.0$ & $100.0$ & $100.0$ & $\phantom{+0}+6$ / $+10$ / $+10$ \\
\bottomrule
\end{tabular}
\end{table}

\begin{table}[h]
\centering
\small
\caption{Add1-B: OGBench-Cube complementarity at varying $T$ on \texttt{cube\_ds5\_ss1}
(success rate, \%). The world-model-only retrieval arm is \texttt{c2}.}
\label{tab:add1-b}
\begin{tabular}{lrrrr}
\toprule
\textbf{Arm} & $T{=}15$ & $T{=}30$ & $T{=}60$ & gap behavior \\
\midrule
baseline (no retrieval)        & $62.0$ & $64.0$ & $66.0$ & N/A \\
WM-only $+$ retrieval arm & $82.0$ & $84.0$ & $84.0$ & N/A \\
full IMWM                       & $90.0$ & $96.0$ & $96.0$ & N/A \\
\midrule
$\Delta$ (IMWM $-$ WM$+$retr.)   & $+8$   & $+12$  & $+12$  & persists \\
\bottomrule
\end{tabular}
\end{table}

The $\text{baseline}\to\text{IMWM}$ gap stabilizes at $+24$--$28$~pp on OGBench-Cube and
$+10$~pp on Two-Room from $T \geq 30$ and does not close at $T{=}60$; the OGBench-Cube
complementarity gap is $+8$ at $T{=}15$ and $+12$ at $T \in \{30,60\}$. The sentinel sweep therefore argues against a simple CEM-iteration-budget explanation.
\emph{Scope:} this varies $T$ only at fixed $S{=}300$; the full $(S,T)$ grid is
left to future work (\S\ref{sec:limitations-future}).

\subsection{Candidate-rank vs.\ success distribution (oracle diagnostic)}
\label{sec:exp-diag-rank-success}

Figure~\ref{fig:diag1-rank-histogram} reports the rank (by the oracle-dynamics
terminal latent-MSE cost) of the first goal-reaching candidate among the
$S{=}300$ CEM candidates, per replan, on the oracle-dynamics diagnostic logs.
On the full 12-cell OGBench-Cube grid ($1{,}200$ records), $77.0\%$ of replans
contain a goal-reaching candidate and $95.6\%$ of those have rank-$0$; on the
full 12-cell Two-Room grid ($1{,}200$ records), $85.2\%$ contain a candidate and
$99.6\%$ of those have rank-$0$; on three auxiliary Reacher cells ($300$
records), $100\%$ of records with a goal-reaching candidate have rank-$0$ (median
rank $0$ on all three tasks). Thus,
when successes exist in the population, the latent objective ranks one first;
together with \S\ref{sec:diagnosis-oracle} this localizes the failure to
candidate coverage, not ranking. The two on-disk Push-T oracle files are
quarantined (Appendix~\ref{app:pusht-quarantine}) and excluded.

\begin{figure}[h]
\centering
\includegraphics[width=0.85\textwidth]{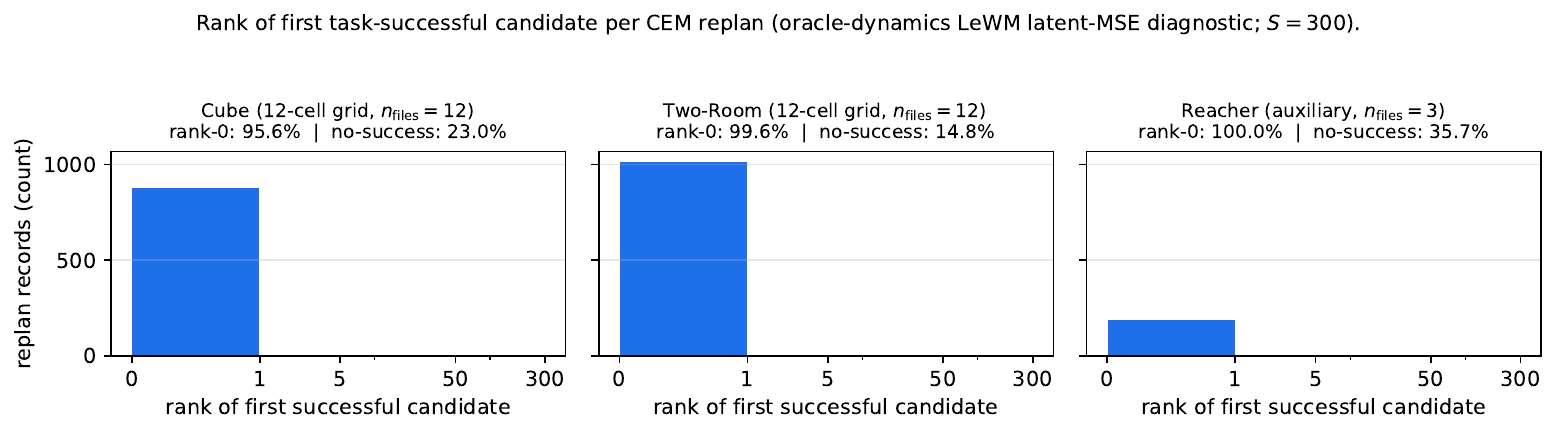}
\caption{Rank of the first goal-reaching candidate per CEM replan under the
oracle-dynamics terminal latent-MSE cost (\S\ref{sec:diagnosis-oracle}),
$S{=}300$. OGBench-Cube and Two-Room full 12-cell grids, Reacher auxiliary
(3 cells); Push-T excluded as quarantined
(Appendix~\ref{app:pusht-quarantine}).}
\label{fig:diag1-rank-histogram}
\end{figure}

\end{document}